\documentclass[11pt, letterpaper]{article}

\usepackage{fullpage}
\usepackage{microtype}
\usepackage{graphicx}
\usepackage{subfigure}
\usepackage{booktabs} %

\usepackage{hyperref}

\usepackage{algorithm}
\usepackage{algorithmic}

\usepackage{amssymb,amsmath,amsthm,amscd}
\newenvironment{prevproof}[2]{\noindent {\bf {Proof of {#1}~\ref{#2}.}}}{$\hfill\qed$\vskip \belowdisplayskip}

\theoremstyle{definition}

\newtheorem{theorem}{Theorem}[section]
\newtheorem{definition}[theorem]{Definition}

\newtheorem{corollary}[theorem]{Corollary}

\newtheorem{lemma}[theorem]{Lemma}

\newtheorem{conjecture}[theorem]{Conjecture}

\newtheorem{remark}[theorem]{Remark}

\newtheorem{invariant}[theorem]{Invariant}

\def\t{\widetilde}

\usepackage{todonotes}

\begin{document}
\clubpenalty=10000
\widowpenalty = 10000

\title{Sparse Convex Optimization via Adaptively Regularized Hard Thresholding}
\author{
Kyriakos Axiotis\thanks{MIT, \tt{kaxiotis@mit.edu}.}
\and
Maxim Sviridenko\thanks{Yahoo! Research, {\tt sviri@verizonmedia.com}.}}

\date{}
\maketitle

\begin{abstract}

The goal of \emph{Sparse Convex Optimization} is to optimize a convex function $f$
under a sparsity constraint $s\leq s^*\gamma$, where $s^*$ is the target number of non-zero entries in a feasible solution (sparsity)
and $\gamma\geq 1$ is an approximation factor.
There has been a lot of work to analyze the sparsity guarantees of various algorithms
(LASSO, Orthogonal Matching Pursuit (OMP), Iterative Hard Thresholding (IHT)) in terms of
the \emph{Restricted Condition Number} $\kappa$. The best known algorithms guarantee to find an approximate solution of value $f(x^*)+\epsilon$ with the sparsity bound of
$\gamma = O\left(\kappa\min\left\{\log \frac{f(x^0)-f(x^*)}{\epsilon}, \kappa\right\}\right)$, where $x^*$
is the target solution.
We present a new \emph{Adaptively Regularized Hard Thresholding (ARHT)} algorithm that makes significant progress on this problem
by bringing the bound down to $\gamma=O(\kappa)$, which has been shown to be tight for a general
class of algorithms including LASSO, OMP, and IHT. This is achieved without significant sacrifice in the runtime efficiency compared to the fastest known algorithms.
We also provide a new analysis of OMP with Replacement (OMPR) for general $f$, under the condition
$s > s^* \frac{\kappa^2}{4}$, which yields
Compressed Sensing bounds under the Restricted Isometry Property (RIP).
When compared to other Compressed Sensing approaches, it has the advantage of providing a strong tradeoff
between the RIP condition and the solution sparsity, while working for any general function $f$ that meets
the RIP condition.

\end{abstract}

\section{Introduction}

\emph{Sparse Convex Optimization} is the problem of optimizing a convex objective, while constraining the sparsity of the solution (its number of non-zero entries).
Variants and special cases of this problem have been studied for many years, and there have been countless applications in Machine Learning, Signal Processing, and Statistics.
In Machine Learning it is used to regularize models by enforcing parameter sparsity,
since a sparse set of parameters often leads to better model generalization. Furthermore, in a lot of large scale applications the number of parameters of a trained model
is a significant factor in computational efficiency, thus improved sparsity can lead to improved time and memory performance.
In applied statistics, a single extra feature translates to a real cost from increasing the number of samples. 
In Compressed Sensing, finding a sparse solution to a Linear Regression problem can be used to significantly reduce the sample size
for the recovery of a target signal.
In the context of these applications, decreasing sparsity by even a small amount while not increasing the accuracy can have a significant impact.

\paragraph{Sparse Optimization}
Given a function $f: \mathbb{R}^n\rightarrow\mathbb{R}$ and any $s^*$-sparse \emph{(unknown) target solution} $x^*$,
the Sparse Optimization problem is to find an $s$-sparse solution $x$, i.e. a solution with at most $s$ non-zero entries,
such that $f(x) \leq f(x^*) + \epsilon$
and $s \leq s^* \gamma$, where $\epsilon > 0$ is a desired accuracy and $\gamma\geq 1$ is an approximation factor
for the target sparsity. Even if $f$ is a convex function, the sparsity constraint makes this problem non-convex,
and it has been shown that it is an intractable problem, even when $\gamma = O\left(2^{\log^{1-\delta}n}\right)$
and $f$ is the Linear Regression objective \cite{Natarajan95, FKT15}.
However, this worst-case behavior is not observed in practice,
and so a large body of work has been devoted to the analysis of algorithms
under the assumption that the \emph{restricted condition number} $\kappa_{s+s^*}=\frac{\rho_{s+s^*}^+}{\rho_{s+s^*}^-}$ (or just $\kappa=\frac{\rho^+}{\rho^-}$) of $f$ 
is bounded \cite{Natarajan95,SSZ10,Zhang11,BRB13,LYF14,JTK14,YLZ16,SL17,SL17_2,JTK14,SGJN18}.
Note: Here, $\rho_{s+s^*}^+$ is the maximum smoothness constant of any restriction of $f$ on an $(s+s^*)$-sparse subset of coordinates
and $\rho_{s+s^*}^-$ is the minimum strong convexity constant of any restriction of $f$ on an $(s+s^*)$-sparse subset of coordinates.

The first algorithm for this problem, often called \emph{Orthogonal Matching Pursuit (OMP)} or \emph{Greedy}, was analyzed by \cite{Natarajan95} for Linear Regression,
and subsequently for general $f$ by \cite{SSZ10}, obtaining the
guarantee that the sparsity of the returned solution is $O\left(s^*\kappa \log \frac{f(x^0) - f(x^*)}{\epsilon}\right)$
\footnote{Even though \cite{Natarajan95} states a less general result, this is what is implicitly proven.}.
In applications where having low sparsity is crucial, the dependence of sparsity on the required accuracy $\epsilon$ is undesirable.
The question of whether this dependence can be removed was answered positively \cite{SSZ10,JTK14} giving a sparsity guarantee of
$O(s^*\kappa^2)$. As remarked in \cite{SSZ10}, this bound sacrifices the linear dependence on $\kappa$, 
while removing the dependence on $\epsilon$ and $f(x^0)-f(x^*)$.

Since then, there has been some work on improving these results by introducing non-trivial assumptions, such as the target solution $x^*$ being close to globally optimal. 
More specifically, 
\cite{Zhang11} defines the 
\emph{Restricted Gradient Optimal Constant (RGOC) at level $s$}, $\zeta_{s}$ (or just $\zeta$)
as the $\ell_2$ norm of the top-$s$ elements in $\nabla f(x^*)$
and analyzes an algorithm that gives sparsity
$s = O\left(s^*\kappa \log\left(s^*\kappa\right)\right)$, and 
such that $f(x) \leq f(x^*) + O(\zeta^2 / \rho^-)$. \cite{SGJN18} 
strengthens this bound to $f(x) \leq f(x^*) + O(\zeta^2 / \rho^+)$
with sparsity $s = O\left(s^*\kappa \log \kappa\right)$.
However, this means that $f(x)$ might be much larger than $f(x^*) + \epsilon$ 
in general.
To the best of our knowledge, no improvement has been made over the
$O\left(s^*\min\left\{\kappa\frac{f(x^0)-f(x^*)}{\epsilon}, \kappa^2\right\}\right)$ bound in the general case.

Another line of work studies a maximization version of the sparse convex optimization problem 
as well as its generalizations for matroid constraints \cite{altschuler2016greedy,elenberg2017streaming,chen2017weakly}.

\paragraph{Sparse Solution and Support Recovery}

Often, as is the case in Compressed Sensing, one needs a guarantee on the closeness of the solution $x$ to the target solution $x^*$ in absolute terms,
rather than in terms of the value of $f$.
The goal is usually either to recover (a superset of) the target support,
or to ensure that the returned solution is close to the target solution in $\ell_2$ norm.
The results for this problem either assume 
a constant upper bound on the \emph{Restricted Isometry Property (RIP)} constant $\delta_r := \frac{\kappa_r - 1}{\kappa_r + 1}$ for some $r$ (RIP-based recovery),
or that $x^*$ is close to being a global optimum (RIP-free recovery).
This problem has been extensively studied and is an active research area in the vast Compressed Sensing literature.
See also the survey by \cite{BCKV15}.

In the seminal papers of \cite{CT05,CRT06,Donoho06,Candes08} it was shown that for the Linear Regression problem
when $\delta_{2s^*}<\sqrt{2}-1\approx 0.41$, 
the LASSO algorithm \cite{LASSO}
can recover a solution with $\left\Vert x-x^*\right\Vert_2^2 \leq C f(x^*)$,
where $C$ is a constant depending only on $\delta_{2s^*}$
and $f(x^*) = \frac{1}{2} \left\Vert Ax^* - b\right\Vert_2^2$ is the error of the target solution\footnote{$f(x^*)$ is also commonly denoted as $\frac{1}{2}\|\eta\|_2^2$, where
$Ax^* = b + \eta$, i.e. $\eta$ is the measurement \emph{noise}.}.
Since then, a multitude of results of similar flavor have appeared, either
giving related guarantees for the LASSO algorithm while improving the RIP upper bound \cite{FL09,CWX09,Foucart10,cai2010new,mo2011new,l1_best}
which culminate in a bound of $\delta_{2s^*} < 0.6248$,
or showing that similar guarantees can be obtained by greedy algorithms under more restricted RIP conditions, but that 
are typically faster than LASSO \cite{ROMP_noiseless,ROMP,CoSaMP,IHT,JTD11,foucart2011hard,foucart2012sparse}. 
See also the comprehensive surveys \cite{foucart2017mathematical, mousavi2019survey}.

\cite{CoSaMP} presents a greedy algorithm called \emph{CoSaMP} and shows that for Linear Regression it achieves a bound in the form of \cite{Candes08}
while having a more efficient implementation. Their method works for the more restricted
RIP upper bound of $\delta_{2s^*} < 0.025$, or $\delta_{4s^*} < 0.4782$ as improved by \cite{foucart2017mathematical}.
\cite{IHT} proves that another greedy algorithm called \emph{Iterative Hard Thresholding (IHT)} achieves a similar bound to that of CoSaMP for Linear Regression, 
with the condition
$\delta_{3s^*} < 0.067$, which is improved to $\delta_{2s^*} < \frac{1}{3}$ by \cite{JTD11} and to $\delta_{3s^*} < 0.5774$ by \cite{foucart2011hard}.

The RIP-free line of research has shown that strong results can be achieved without a RIP upper bound, given that
the target solution is sufficiently close to being a global optimum. These results typically require
that $s$ is significantly larger than $s^*$. In particular, 
\cite{Zhang11} shows that 
if $\zeta$ is the RGOC of $f$
it can be guaranteed that 
$\left\Vert x - x^*\right\Vert_2 \leq 2\sqrt{6} \frac{\zeta}{\rho^-}$
(or $(1+\sqrt{6}) \frac{\zeta}{\rho^-}$ with a slightly tighter analysis).
\cite{SGJN18} strengthens this bound to $\left(1+\sqrt{1 + \frac{5}{\kappa}}\right) \frac{\zeta}{\rho^-}$.
Furthermore, it has been shown that as long as a ``Signal-to-Noise'' condition holds,
one can actually recover a superset of the target support. 
Typically the condition is a lower bound on $|x_{\min}^*|$, the minimum magnitude non-zero
entry of the target solution. Different lower bounds that have been devised
include $\Omega\left(\frac{\sqrt{s+s^*} \left\Vert \nabla f(x^*)\right\Vert_\infty}{\rho_{s+s^*}^-}\right)$ \cite{JTK14},
which was later improved to $\Omega\left(\sqrt{\frac{f(x^*) - f(\overline{x}^*)}{\rho_{2s}^-}}\right)$,
where $\overline{x}^*$ is an optimal $s$-sparse solution \cite{YLZ16}. Finally, \cite{SGJN18} 
improves the sparsity bound to $O(s^*\kappa\log\left(s^*\kappa\right))$ in the statistical setting
and \cite{SL17_2} shows that the sparsity can be brought down to
$s = s^* + O(\kappa^2)$ if 
a stronger lower bound of 
$\Omega\left(\sqrt{\kappa} \frac{\zeta}{\rho}\right)$ 
is assumed.

\subsection{Our work}
\label{our_work}

In this work we present a new 
algorithm called \emph{Adaptively Regularized Hard Thresholding (ARHT)}, that closes the longstanding gap between 
the 
$O\left(s^*\kappa\frac{f(x^0)-f(x^*)}{\epsilon}\right)$ and $O\left(s^*\kappa^2\right)$ bounds by getting a sparsity of 
$O(s^*\kappa)$
and thus achieving the best
of both worlds.
As \cite{FKT15} shows that
for a general class of algorithms (including greedy algorithms like OMP, IHT as well as LASSO)
the linear dependence on $\kappa$ is necessary even for the special case of Sparse Regression, 
our result is tight for this class of algorithms.
In Section~\ref{sec:lower_bound} we briefly describe this example and also state a conjecture that it can be turned
into an inapproximability result in Conjecture~\ref{conjecture}.
Furthermore, in Section~\ref{OMPR_lower_bound} we show that the $O(s^*\kappa^2)$ sparsity bound 
is tight for OMPR, thus highlighting the importance of regularization in our method.
Our algorithm is efficient, as it 
requires roughly $O\left(s \log^3 \frac{f(x^0) - f(x^*)}{\epsilon}\right)$ iterations,
each of which includes one function minimization in a restricted support of size $s$
and is simple to describe and implement. Furthermore, it directly implies non-trivial 
results in the area of Compressed Sensing.

We also provide a new analysis of OMPR \cite{JTD11}
and show that 
under the condition that $s > s^*\frac{\kappa^2}{4}$, or equivalently
under the RIP condition $\delta_{s+s^*} < \frac{2\sqrt{\frac{s}{s^*}} - 1}{2\sqrt{\frac{s}{s^*}}+1}$,
it is possible to approximately minimize the function $f$ up to some error depending on the RIP constant
and the closeness of $x^*$ to global optimality. More specifically, we show that for any $\epsilon > 0$
OMPR returns a solution $x$ such that
\begin{align*}
f(x) \leq f(x^*) + \epsilon + C_1 (f(x^*) - f(x^{\mathrm{opt}}))
\end{align*}
where $x^{\mathrm{opt}}$ is the globally optimal solution,
as well as 
\begin{align*}
\left\Vert x - x^*\right\Vert_2^2 \leq \epsilon + C_2 (f(x^*) - f(x^{\mathrm{opt}}))
\end{align*}
where $C_1,C_2$ are constants that only depend on %
$\frac{s}{s^*}$ and $\delta_{s+s^*}$. An important feature of our approach is that 
it provides a 
meaningful tradeoff between the RIP constant upper bound and the sparsity of the solution,
even when the sparsity $s$ is arbitrarily close to $s^*$. In other words,
one can relax the RIP condition at the expense of increasing the sparsity of the returned solution.
Furthermore, our analysis applies to general functions with bounded RIP constant.

Experiments with real data suggest that ARHT and a variant of OMPR which we call \emph{Exhaustive Local Search} 
achieve promising performance in recovering sparse solutions.

\subsection{Comparison to previous work}
\label{comparison}

\paragraph{Sparse Optimization}

Our Algorithm~\ref{local_reg} (ARHT) returns a solution with $s = O(s^*\kappa)$ without any additional assumptions,
thus significantly improving over the bound 
$O\left(s^*\min\left\{\kappa\frac{f(x^0)-f(x^*)}{\epsilon}, \kappa^2\right\}\right)$ 
that was known in previous work.
This proves that neither any dependence on the required solution accuracy $\epsilon$, nor 
the quadratic dependence on the condition number $\kappa$ is necessary. Furthermore, no assumption on the function or the target solution is required to achieve
this bound. Importantly, previous results imply that our bound is tight up to constants
for a general class of algorithms, including Greedy-type algorithms and LASSO \cite{FKT15}.

\paragraph{Sparse Solution Recovery}

In Corollary~\ref{restricted_gradient_corollary}, we show that
the improved guarantees of Theorem~\ref{reg_theorem} immediately imply that
ARHT gives a bound of $\left\Vert x - x^*\right\Vert_2 \leq (2+\theta) \frac{\zeta}{\rho^-}$
for any $\theta > 0$, where $\zeta$ is the Restricted Gradient Optimal Constant. This improves the constant factor in front
of the corresponding results of \cite{Zhang11,SGJN18}.

As we saw, our Theorem~\ref{local_theorem} directly implies that OMPR %
gives an upper bound on $\left\Vert x-x^*\right\Vert_2^2$ 
of the same form as the RIP-based bounds in previous work, under the condition
$\delta_{s+s^*} < \frac{2\sqrt{\frac{s}{s^*}} - 1}{2\sqrt{\frac{s}{s^*}} + 1}$.
While previous results either concentrate on the case $s=s^*$,
or $s\gg s^*$,
our analysis provides a way to trade off increased sparsity for a more relaxed RIP bound, allowing
for a whole family of RIP conditions where $s$ is arbitrarily close to $s^*$.
Specifically, if we set $s=s^*$ our work implies recovery for $\delta_{2s^*} < \frac{1}{3}\approx 0.33$, which matches
the best known bound for any greedy algorithm \cite{JTD11}, although it is a stricter
condition than the $\delta_{2s^*} < 0.62$ required by LASSO \cite{foucart2017mathematical}.
Table~\ref{tradeoff} contains a few such RIP bounds implied by our analysis
and shows that it readily surpasses the 
bounds for Subspace Pursuit $\delta_{3s^*} < 0.35$,
CoSaMP $\delta_{4s^*} < 0.48$, 
and OMP $\delta_{31s^*} < 0.33$ \cite{JTD11,Zhang11}.
Another important feature compared to previous work is that all our guarantees are not restricted to Linear Regression
and are true for any function $f$, as long as it satisfies the required RIP condition, which makes the result more general.

\begin{table}
\caption{\label{tradeoff} Compressed Sensing tradeoffs implied by Theorem~\ref{local_theorem}: Sparsity vs RIP condition}
\vskip 0.15in
\begin{center}
\begin{small}
\begin{sc}
\begin{tabular}{cc}
\toprule
$s$ & RIP condition\\
\midrule
$s^*$ & $\delta_{2s^*} < 0.33$\\
$2s^*$ & $\delta_{3s^*} < 0.47$\\
$3s^*$ & $\delta_{4s^*} < 0.55$\\
$30s^*$ & $\delta_{31s^*} < 0.83$\\
\bottomrule
\end{tabular}
\end{sc}
\end{small}
\end{center}
\vskip -0.1in
\end{table}

\paragraph{Sparse Support Recovery}

Corollary~\ref{support_recovery} shows that
as a direct consequence of our work,
the condition
$|x_{\min}^*| > \frac{\zeta}{\rho^-}$ suffices for our algorithm to recover a superset 
of the support with size $s = O(s^* \kappa)$. 
Compared to \cite{JTK14}, we improve both the size of the superset, as well as the condition, since 
$\sqrt{s} \frac{\left\Vert \nabla f(x^*)\right\Vert_\infty}{\rho^-} \geq \sqrt{\frac{s}{s^*}} \frac{\zeta}{\rho^-} = \Omega\left(\frac{\zeta}{\rho^-}\right)$.
Compared to \cite{SL17_2}, the bounds on the superset size are incomparable in general, but our $|x_{\min}^*|$ condition is more relaxed, since
$\sqrt{\kappa}\frac{\zeta}{\rho^-} = \Omega(\frac{\zeta}{\rho^-})$.
Finally, compared to \cite{YLZ16} 
we have a stricter lower bound on $|x_{\min}^*|$, but with a better bound on the superset size ($O(s^*\kappa)$ instead of $O(s^*\kappa^2)$).
Although not explicitly stated, \cite{Zhang11,SGJN18} also
give a similar lower bound of
$\sqrt{1 + \frac{10}{\kappa}}\frac{\zeta}{\rho^-}$ which we improve by a constant factor.

\paragraph{Runtime discussion}

ARHT has the advantage of being very simple to implement in practice.
The runtime of Algorithm~\ref{local_reg} (ARHT) is 
comparable to that of the most efficient greedy algorithms (e.g. OMP/OMPR),
as it requires a single function minimization per iteration.
On the other hand, Algorithm~\ref{exlocal} (Exhaustive Local Search) is less efficient, as it requires $O(n)$ 
function minimizations in each iteration, although in practice one might be able to speed it up by
exploiting the fact that the problems being solved in each iteration are very closely related.

\paragraph{Naming Conventions}
The algorithm that we call 
\emph{Orthogonal Matching Pursuit (OMP)}, 
is also known as 
``Greedy'' \cite{Natarajan95}, ``Fully Corrective Forward Greedy Selection'' or just ``Forward Selection''.
What we call \emph{Orthogonal Matching Pursuit with Replacement (OMPR)} \cite{JTD11} is also known by various other names.
It is referenced in \cite{SSZ10}
as a simpler variant of their ``Fully Corrective Forward Greedy Selection with Replacement'' algorithm,
or just Forward Selection with Replacement,
or ``Partial Hard Thresholding with parameter $r=1$ ($\mathrm{PHT}(r)$ where $r=1$)'' \cite{JTD17} which
is a generalization of Iterative Hard Thresholding.
Finally, what we call \emph{Exhaustive Local Search} is essentially a variant of ``Orthogonal Least Squares'' that includes replacement steps,
and also appears in \cite{SSZ10} as ``Fully Corrective Forward Greedy Selection with Replacement'', or just ``Forward Stepwise Selection with Replacement''.
See also \cite{BD07} regarding naming conventions.

\begin{remark}
Most of the results in the literature either only apply to, or
are only presented for the Linear Regression problem. Since all
of our results apply to general function minimization, we present them as such.
\end{remark}

\section{Preliminaries}
\subsection{Definitions}
We denote $[i] := \{1,2,\dots,i\}$. For any $x\in\mathbb{R}^n$ and $R\subseteq[n]$, we define $x_R\in\mathbb{R}^n$ as
\begin{align*}
(x_R)_i = \begin{cases}
x_i & i\in R\\
0 & \text{ otherwise }\\
\end{cases}
\end{align*}
Additionally, for any differentiable function $f:\mathbb{R}^n\rightarrow\mathbb{R}$ with gradient $\nabla f(x)$, we will denote by
$\nabla_R f(x)$ the restriction of $\nabla f(x)$ to $R$, i.e. $\left(\nabla f(x)\right)_R$.

\begin{definition}[$\ell_p$ Norms]
For any $p\in\mathbb{R}_{+}$, we define
\[ \left\Vert x\right\Vert_p = \left(\sum\limits_i |x_i|^p\right)^{1/p} \]
as well as the special cases
$\left\Vert x\right\Vert_0 = |\{i\ :\ x_i \neq 0\}|$ and $\left\Vert x\right\Vert_\infty = \underset{i}{\max}\ |x_i| $
\end{definition}

\begin{definition}
For any $x\in\mathbb{R}^n$, we denote the \emph{support} of $x$ by
$ \mathrm{supp}(x) = \{i\ :\ x_i\neq 0\} $
\end{definition}

\begin{definition}[Restricted Condition Number]
Given a differentiable
function $f$,
the \emph{Restricted $\ell_2$-Smoothness (RSS)} constant, or just Restricted Smoothness constant,
of $f$ at sparsity level $s$ is the minimum $\rho_s^+\in\mathbb{R}$
such that 
\[ f(y) \leq f(x) + \nabla f(x)^\top(y-x) + \frac{\rho_s^+}{2} \left\Vert y - x\right\Vert_2^2 \]
for all $x,y\in\mathbb{R}^n$ with $|\mathrm{supp}(y-x)| \leq s$.
Similarly,
the \emph{Restricted $\ell_2$-Strong Convexity (RSC)} constant, or just Restricted Strong Convexity constant, 
of $f$ at sparsity level $s$ is the maximum $\rho_s^-\in\mathbb{R}_{+}$
such that 
\[ f(y) \geq f(x) + \nabla f(x)^\top(y-x) + \frac{\rho_s^-}{2} \left\Vert y - x\right\Vert_2^2 \]
for any $x,y\in\mathbb{R}^n$ with $|\mathrm{supp}(y-x)| \leq s$.
Given that $\rho_s^+,\rho_s^->0$, the \emph{Restricted Condition Number} of $f$ at sparsity level $s$ is defined as
$\kappa_s = \rho_s^+/\rho_s^-$. We will also make use of $\widetilde{\kappa}_s = \rho_2^+ / \rho_s^-$ which is
at most $\kappa_s$ as long as $s\geq 2$.
\end{definition}

The following lemma stems from the definitions of $\rho_2^+,\rho_1^+$ and can be used to relate $\rho_2^+$ with $\rho_1^+$
\begin{lemma}
For any function $f$ that has the RSC property at sparsity level $\geq 2$
and RSS constants 
$\rho_1^+,\rho_2^+$
at sparsity levels $1$ and $2$ respectively, we have 
$\rho_2^+ \leq 2\rho_1^+$.
\label{rho1}
\end{lemma}
\begin{proof}
For any $x,y\in\mathbb{R}^n$ such that $\left|\mathrm{supp}(y-x)\right| \leq 2$, 
We will prove that 
\begin{align*}
f(y) \leq f(x) + \langle \nabla f(x), y-x\rangle + \frac{2\rho_1^+}{2} \left\Vert y-x\right\Vert_2^2
\end{align*}
Let 
$y = x + \alpha \vec{1}_i + \beta \vec{1}_j$ for some $i,j\in[n]$ and $\alpha,\beta\in\mathbb{R}$. 
We assume $i\neq j$ and since otherwise the claim already follows from RSS at sparsity level $1$.
We apply the RSS property with sparsity level $1$ to get the inequalities
\begin{align*}
 f(x + 2\alpha \vec{1}_i) \leq f(x) + 2\langle \nabla f(x),\alpha \vec{1}_i\rangle + 4\frac{\rho_1^+}{2} \left\Vert \alpha \vec{1}_i \right\Vert_2^2
\end{align*}
and
\begin{align*}
f(x + 2\beta\vec{1}_j) \leq f(x) + 2\langle \nabla f(x), \beta \vec{1}_j\rangle + 4\frac{\rho_1^+}{2} \left\Vert \beta \vec{1}_j \right\Vert_2^2
\end{align*}
Now, by using convexity (more precisely restricted convexity at sparsity level $2$ that is implied by RSC) we have 
\begin{align*}
 f(y) &= f(x + \alpha \vec{1}_i + \beta\vec{1}_j)\\
& \leq \frac{1}{2} \left(f(x+2\alpha \vec{1}_i) + f(x + 2\beta\vec{1}_j)\right)\\
& \leq f(x) + \langle \nabla f(x),\alpha \vec{1}_i + \beta \vec{1}_j\rangle + \frac{2\rho_1^+}{2} \left\Vert \alpha \vec{1}_i + \beta \vec{1}_j\right\Vert_2^2\\
& = f(x) + \langle \nabla f(x),y-x\rangle + \frac{2\rho_1^+}{2} \left\Vert y-x\right\Vert_2^2
\end{align*}
\end{proof}

\begin{definition}[Restricted Isometry Property (RIP)]
We say that a differentiable function $f$ has the \emph{Restricted Isometry Property} at sparsity level $s$ if
$\rho_s^+,\rho_s^->0$, and
the \emph{RIP constant} of $f$ at sparsity level $s$ is then defined as $\delta_s = \frac{\kappa_s - 1}{\kappa_s + 1}$.\footnote{We note that 
this is a more general definition than the one usually given for quadratic functions (i.e. Linear Regression).}
\end{definition}

\begin{definition}[Restricted Gradient Optimal Constant (RGOC)]
Given a differentiable function $f$ and a ``target'' solution $x^*$,
the \emph{Restricted Gradient Optimal Constant} \cite{Zhang11} at sparsity level $s$
is the minimum $\zeta_s\in\mathbb{R}_+$ such that
\[ \left|\langle\nabla f(x^*), y\rangle \right| \leq \zeta_s \left\Vert y\right\Vert_2 \]
for all $s$-sparse $y$.
Setting $y = \nabla_{S} f(x^*)$ for some set $S$ with $|S| \leq s$, this implies that
$\zeta_s \geq \left\Vert \nabla_S f(x^*)\right\Vert$. An alternative definition is
that $\zeta_s$ is the $\ell_2$ norm of the $s$ elements of $\nabla f(x^*)$ with highest
absolute value.
\label{RGOC}
\end{definition}
\begin{definition}[$S$-restricted minimizer]
Given 
$f:\mathbb{R}^n\rightarrow \mathbb{R}$, 
$x^*\in\mathbb{R}^n$, and 
$S\subseteq[n]$, 
we will call $x^*$ an $S$-restricted minimizer of $f$ if
$\mathrm{supp}(x^*)\subseteq S$ 
and for all $x$ such that 
$\mathrm{supp}(x)\subseteq S$ we have
$f(x^*)\leq f(x)$.
\end{definition}

In Lemma~\ref{martingale_concentration} we state a standard martingale concentration inequality
that we will use. See also \cite{CL06} for more on martingales.

\begin{lemma}[Martingale concentration inequality (Special case of Theorem~6.5 in~\cite{CL06})]
\label{martingale_concentration}
Let $Y_0=0,Y_1,\dots, Y_n$ be a martingale with respect to the sequence $i_1,\dots,i_n$
such that
\[ \mathrm{Var}\left(Y_k\ |\ i_1,\dots,i_{k-1}\right) \leq \sigma^2 \]
and
\[ Y_{k-1} -Y_{k} \leq M \]
for all $k\in[n]$, then for any $\lambda > 0$,
\[ \Pr\left[Y_n \leq -\lambda\right] \leq e^{-\lambda^2/\left(2\left(n \sigma^2 + M\lambda /3\right)\right)} \]
\end{lemma}

\subsection{Algorithms}
\subsubsection{$\ell_1$ optimization (LASSO)}
The LASSO approach is to relax the $\ell_0$ constraint into an $\ell_1$ one, thus solving the following optimization problem:
\begin{align}
    &\underset{x}{\min}\ f(x) + \lambda \left\Vert x\right\Vert_1
\label{l1}
\end{align}
for some parameter $\lambda > 0$.

\subsubsection{Iterative Hard Thresholding (IHT):}
\cite{IHT} define the hard thresholding operator $H_r(x)$ as
\[ [H_r(x)]_i = \begin{cases}
x_i & \text{if $|x_i|$ is one of the $r$ entries of $x$}\\
	& \text{with largest magnitude}\\
0 & \text{otherwise}
\end{cases}
\]
Using this, the algorithm is described in Algorithm~\ref{IHT}.
\begin{algorithm}[h!]
\caption{Iterative Hard Thresholding (IHT)}\label{IHT}
\begin{algorithmic}[1]
\FUNCTION{$\mathrm{IHT}(s,T)$}
\STATE function to be minimized $f:\mathbb{R}^n\rightarrow\mathbb{R}$
\STATE number of iterations $T$
\STATE output sparsity $s$
\STATE $S^0 \leftarrow \emptyset$
\STATE $x^0 \leftarrow \vec{0}$
\FOR {$t=0\dots T-1$}
\STATE $x^{t+1} \leftarrow H_s\left(x^{t} - \eta \nabla f(x^{t})\right)$
\ENDFOR
\STATE {\bf return} $x^T$
\ENDFUNCTION
\end{algorithmic}
\end{algorithm}

\subsubsection{Orthogonal Matching Pursuit (Greedy/OMP/Fwd stepwise selection)}
The algorithm is described in Algorithm~\ref{greedy}.

\begin{algorithm}[h!]
\caption{Greedy/OMP/Fwd stepwise selection}\label{greedy}
\begin{algorithmic}[1]
\FUNCTION{$\mathrm{greedy}(s)$}
\STATE function to be minimized $f:\mathbb{R}^n\rightarrow\mathbb{R}$
\STATE output sparsity $s$
\STATE $x^0 \leftarrow \vec{0}$
\FOR {$t=0\dots s-1$}
\STATE $i \leftarrow \mathrm{argmax}\big\{|\nabla_i f(x^{t})|\ \big|\  i\in [n]\backslash S^{t}\big\}$
\STATE $S^{t+1}  \leftarrow S^{t} \cup \{i\}$
\STATE $x^{t+1} \leftarrow \mathrm{argmin}\big\{f(x)\ \big|\ \mathrm{supp}(x)\subseteq S^{t+1}\big\}$
\ENDFOR
\STATE {\bf return} $x^s$
\ENDFUNCTION
\end{algorithmic}
\end{algorithm}

\subsubsection{Orthogonal Matching Pursuit with Replacement (Local search/OMPR/Fwd stepwise selection with replacement steps)}
The algorithm is described in Algorithm~\ref{local}.

\begin{algorithm}[h!]
\caption{Orthogonal Matching Pursuit with Replacement}\label{local}
\begin{algorithmic}[1]
\FUNCTION{$\mathrm{OMPR}(s)$}
\STATE function to be minimized $f:\mathbb{R}^n\rightarrow\mathbb{R}$
\STATE output sparsity $s$
\STATE $S^0 \leftarrow [s]$
\STATE $x^0 \leftarrow \mathrm{argmin}\big\{f(x)\ \big|\ \mathrm{supp}(x)\subseteq S^0\big\}$
\STATE $t \leftarrow 0$
\WHILE {{\bf true}}
\STATE $i \leftarrow \mathrm{argmax}\big\{|\nabla_i f(x^{t})|\ \big|\  i\in [n]\backslash S^{t}\big\}$
\STATE $j \leftarrow \mathrm{argmin}\big\{|x_j^{t}|\ \big|\  j\in S^{t}\big\}$
\STATE $S^{t+1}  \leftarrow S^{t} \cup \{i\}\backslash \{j\}$
\STATE $x^{t+1} \leftarrow \mathrm{argmin}\big\{f(x)\ \big|\ \mathrm{supp}(x)\subseteq S^{t+1}\big\}$
\IF {$f(x^{t+1}) \geq f(x^t)$}
\STATE {\bf break}
\ENDIF
\STATE $t \leftarrow t + 1$
\ENDWHILE
\STATE $T \leftarrow t$
\STATE {\bf return} $x^{T}$
\ENDFUNCTION
\end{algorithmic}
\end{algorithm}

\subsubsection{Exhaustive Local Search}
The algorithm in this section is similar to OMPR, in that it iteratively inserts a new
element in the support while removing one from it at the same time. While, as in OMPR,
the element to be removed is the minimum magnitude entry, the one to be inserted is chosen
to be the one which results in the maximum decrease in the value of the objective.
It is described in Algorithm~\ref{exlocal}.

\begin{algorithm}[h!]
\caption{Exhaustive Local Search}\label{exlocal}
\begin{algorithmic}[1]
\STATE function to be minimized $f:\mathbb{R}^n\rightarrow\mathbb{R}$
\STATE target sparsity $s$
\STATE number of iterations $T$
\STATE $S^0 \leftarrow [s]$
\STATE $x^0 \leftarrow \mathrm{argmin}\big\{f(x)\ \big|\ \mathrm{supp}(x)\subseteq S^0\big\}$
\FOR {$t=0\dots T-1$}
\STATE $j \leftarrow \underset{j\in S^{t}}{\mathrm{argmin}}\ x_j^2$
\STATE $i \leftarrow \underset{i\in [n]\backslash S^{t}}{\mathrm{argmin}}
\left\{ 
\underset{x\ :\ \mathrm{supp}(x)\subseteq S^{t}\cup \{i\}\backslash\{j\}}{\min} f(x)\right\}$
\STATE $S^{t+1} \leftarrow S^{t}\cup\{i\}\backslash\{j\}$
\STATE $x^{t+1} \leftarrow \mathrm{argmin}\big\{f(x)\ \big|\ \mathrm{supp}(x)\subseteq S^{t+1}\big\}$
\IF {$f(x^{t+1}) \geq f(x^t)$}
\STATE {\bf return} $x^{t}$
\ENDIF
\ENDFOR
\STATE {\bf return} $x^{T}$
\end{algorithmic}
\end{algorithm}

\begin{remark}
In the following sections, we will denote the minimization objective
by $f$, the RSS and RSC parameters 
$\rho_{2}^+$ and $\rho_{s+s^*}^-$
by $\rho^+$ and $\rho^-$ respectively,
as well as 
$\kappa = \frac{\rho_{s+s^*}^+}{\rho_{s+s^*}^-}$
and $\t{\kappa} = \frac{\rho_2^+}{\rho_{s+s^*}^-}$.
Note that the use of $\rho_2^+$ instead of $\rho_{1}^+$ used in some works is not restrictive. As shown in Lemma~\ref{rho1},
$\rho_2^+ \leq 2 \rho_1^+$ and so in all the bounds involving $\t{\kappa}$,
it can be replaced by $2\frac{\rho_1^+}{\rho_{s+s^*}^-}$, thus only losing a factor of $2$.
Furthermore, we state our results in terms of $\widetilde{\kappa}$ as opposed to $\kappa$.
This is always more general since $\widetilde{\kappa} \leq \kappa$.

When no additional context is provided, we denote current solution by $x$ and the target solution $x^*$, with respective support sets
$S$ and $S^*$ and sparsities $s=|S|$ and $s^*=|S^*|$. 
\end{remark}

\section{Adaptively Regularized Hard Thresholding (ARHT)}
\label{sec:arht}
\subsection{Overview and Main Theorem}

Our algorithm is essentially a hard thresholding algorithm (and more specifically OMPR, also known as PHT(1))
with the crucial novelty that it is applied on an adaptively regularized objective function.
Hard thresholding algorithms maintain a solution $x$ supported on $S\subseteq [n]$, which they iteratively update by
inserting new elements into the support set $S$ and removing the same number of elements from it, in order to preserve the sparsity of $x$. 
More specifically, OMPR makes one insertion and one removal in each iteration.
In order to evaluate the element $i$ to be \emph{inserted} into $S$, 
OMPR uses the fact that, because of smoothness, $\frac{\left(\nabla_i f(x)\right)^2}{2\rho_2^+}$ is a lower bound on the decrease
of $f(x)$ caused by inserting $i$ into the support, and therefore picks $i$ to maximize $\left|\nabla_i f(x) \right|$.
Similarly, in order to evaluate the element $j$ to be \emph{removed} from $S$, 
OMPR uses the fact that $\frac{\rho_2^+}{2} x_j^2$ upper bounds the increase of $f(x)$ caused by setting $x_j=0$, and therefore
picks $j$ to minimize $\left|x_j\right|$.
However, the real worth of $j$ might be as small as $\frac{\rho_{2}^-}{2} x_j^2$, so the upper bound can be loose by a factor of 
$\frac{\rho_2^+}{\rho_{2}^-} \geq 
\frac{\rho_2^+}{\rho_{s+s^*}^-} = \widetilde{\kappa}$.

We eliminate this discrepancy by running the algorithm on the regularized function $g(z) := f(z) + \frac{\rho_2^+}{2} \left\Vert z\right\Vert_2^2$. 
As the restricted condition number of $g$ is now $O(1)$, 
the real worth of a removal candidate $j$ matches the upper bound up to a constant factor.

However, even though $g$ is now well conditioned,
the analysis can only guarantee the quality of the solution in terms of the \emph{original} objective
$f$ if the regularization is \emph{not} applied on elements $S^*$, i.e. 
$\frac{\rho_2^+}{2} \left\Vert x_{R\backslash S^*}\right\Vert_2^2$ for some sufficiently large $R\subseteq[n]$;
if this is the case, a solution with sparsity $O(s^* \widetilde{\kappa})$ can be recovered. 
Unfortunately, there is no way of knowing a priori which
elements not to regularize, as this is equivalent to finding the target solution. 
As a result, the algorithm can get trapped in local minima, which are defined as states in which one iteration of the algorithm does not decrease 
$g(x)$, even though $x$ is a suboptimal solution in terms of $f$ (i.e. $f(x) > f(x^*)$).

The main contribution of this work is to characterize such local minima and devise a procedure that is able to successfully
escape them, thus allowing $x$ to converge to a desired solution for the original objective.

The core algorithm is presented in Algorithm~\ref{arls}.
The full algorithm additionally requires some standard routines like binary search and is presented in
Algorithm~\ref{local_reg}.
\begin{algorithm}[h!]
\caption{Adaptively Regularized Hard Thresholding core routine}\label{arls}
\begin{algorithmic}[1]
\FUNCTION{ARHT\_core($s,\mathrm{opt},\epsilon$)}
	\STATE function to be minimized $f:\mathbb{R}^n\rightarrow\mathbb{R}$
	\STATE target sparsity $s$
	\STATE target value $\mathrm{opt}$ (current guess for the optimal value)
	\STATE target error $\epsilon$
	\STATE Define $g_R(x) := f(x) + \frac{\rho_2^+}{2} \left\Vert x_R\right\Vert_2^2$ for all $R\subseteq[n]$.
	\STATE $R^0 \leftarrow [n]$
	\STATE $S^0 \leftarrow [s]$
	\STATE $x^0 \leftarrow \underset{\mathrm{supp}(x)\subseteq S^0}{\mathrm{argmin}}\ g_{R^0}(x)$
	\STATE $T = 2s \log \frac{f(\vec{0}) - \underset{x}{\min}\, f(x)}{\epsilon}$ (number of iterations)
	\FOR {$t=0\dots T-1$}
	\IF {$\underset{\mathrm{supp}(x)\subseteq S^t}{\min}\ f(x) \leq \mathrm{opt}$}
	\STATE {\bf return} $\underset{\mathrm{supp}(x)\subseteq S^t}{\mathrm{argmin}}\ f(x)$
	\ENDIF
	\STATE $i\leftarrow \underset{i\in[n]}{\mathrm{argmax}}\ |\nabla_i g_{R^t}(x^{t})|$
	\STATE $j\leftarrow \underset{j\in S^{t}}{\mathrm{argmin}}\ |x_j|$
	\STATE {$S^{t+1} \leftarrow S^{t}\cup\{i\}\backslash\{j\}$}
	\STATE $x^{t+1} \leftarrow \underset{\mathrm{supp}(x)\subseteq S^{t+1}}{\mathrm{argmin}}\ g_{R^t}(x)$
	\IF {$g_{R^t}(x^t) -g_{R^t}(x^{t+1}) < \frac{1}{s} \left(g_{R^t}(x^t) - \mathrm{opt}\right)$}
	\STATE {$S^{t+1} \leftarrow S^{t}$}
	\STATE {Sample $i\in R^t$ proportional to $(x_i^t)^2$}
	\STATE {$R^{t+1} \leftarrow R^t \backslash \{i\}$}
	\STATE $x^{t+1} \leftarrow \underset{\mathrm{supp}(x)\subseteq S^{t+1}}{\mathrm{argmin}}\ g_{R^{t+1}}(x)$
	\ENDIF
	\ENDFOR
	\STATE {\bf return} $x^{T}$
\ENDFUNCTION
\end{algorithmic}
\end{algorithm}

\begin{algorithm}[h!]
\caption{Adaptively Regularized Hard Thresholding}\label{local_reg}
\begin{algorithmic}[1]
\FUNCTION{ARHT\_robust($s,\mathrm{opt}, \epsilon, B$)}
\STATE function to be minimized $f:\mathbb{R}^n\rightarrow\mathbb{R}$
\STATE lower bound on target value $B$
\STATE $x^{\mathrm{ret}}\leftarrow \vec{0}$
\FOR {$z=1\dots 5\log \left(6n\log\frac{f(\vec{0}) - B}{\epsilon}\right)$}
\STATE $x\leftarrow \mathrm{ARHT\_core}(s,\mathrm{opt},\epsilon)$
\IF {$f(x) < f(x^{\mathrm{ret}})$}
\STATE $x^{\mathrm{ret}} \leftarrow x$
\ENDIF
\ENDFOR
\STATE {\bf return} $x^{\mathrm{ret}}$
\ENDFUNCTION
\FUNCTION{ARHT($s,\epsilon$)}
	\STATE function to be minimized $f:\mathbb{R}^n\rightarrow\mathbb{R}$
	\STATE target sparsity $s$
	\STATE target error $\epsilon$
	\STATE $B\leftarrow \underset{x}{\min}\ f(x)$
	\STATE $l\leftarrow B$
	\STATE $b \leftarrow \vec{0}$
	\STATE $r\leftarrow f(b)$
	\WHILE {$r-l > \epsilon$}
	\STATE $m\leftarrow \frac{l+r}{2}$
	\STATE $x \leftarrow \mathrm{ARHT\_robust}(s, m, \epsilon/3, B)$
	\IF {$f(x) > m + \epsilon/3$}
	\STATE $l \leftarrow m$
	\ELSE
	\STATE $b \leftarrow x$
	\STATE $r \leftarrow f(x)$
	\ENDIF
	\ENDWHILE
	\STATE {\bf return} $b$
\ENDFUNCTION
\end{algorithmic}
\end{algorithm}

In the following, we will let $\mathrm{opt}$ denote a guess on the target value $f(x^*)$. Also, $x^0$ will denote the initial solution, which is an $S^0$-restricted minimizer 
for an arbitrary 
set $S^0\subseteq [n]$ with $|S^0|=s$. In Algorithm~\ref{arls}, $S^0$ is defined explicitly as $[s]$, however in practice one might 
want to pick a better initial set (e.g. returned by running OMP).

We are now ready for stating the main result of this section.
\begin{theorem}
Given a function $f$ and an (unknown) $s^*$-sparse solution $x^*$, with probability at least $1 - \frac{1}{n}$ Algorithm~\ref{local_reg} returns an $s$-sparse solution
$x$ with $f(x) \leq f(x^*) + \epsilon$, as long as
$s \geq s^* \max\{4\widetilde{\kappa} + 7, 12\widetilde{\kappa} + 6\}$.
The number of iterations is 
$O\left(s\log^2 \frac{f(\vec{0})-B}{\epsilon}
	\log \left(n\log \frac{f(\vec{0})-B}{\epsilon}\right)\right)$
where $B = \underset{x}{\min}\ f(x)$.
\label{reg_theorem}
\end{theorem}

The following corollary that bounds the total runtime can be immediately extracted. Note that in practice the total runtime heavily depends on the choice of $f$,
and it can often be improved for various special cases (e.g. linear regression).
\begin{corollary}[Theorem~\ref{reg_theorem} runtime]
If we denote by $G$ the time needed to compute $\nabla f$ and by $M$ the time to minimize $f$ in a restricted subset of $[n]$ of size $s$,
the total runtime of Algorithm~\ref{local_reg} is
$O\left((G+M) s \log^2 \frac{f(\vec{0})-B}{\epsilon}
	\log \left(n\log \frac{f(\vec{0})-B}{\epsilon}\right)\right)$. If gradient descent is used for the implementation of the inner optimization problem,
	then $M = O\left(G \t{\kappa} \log \frac{f(\vec{0})-B}{\epsilon}\right)$ and so 
	the total runtime can be bounded by
$O\left(G s \t{\kappa} \log^3 \frac{f(\vec{0})-B}{\epsilon}
\log \left(n\log \frac{f(\vec{0})-B}{\epsilon}\right)\right)$. 
\end{corollary}

Before proving the above theorem, we provide the main components that are needed for its proof.
It is important to split the iterations of Algorithm~\ref{arls} into two categories:
Those that make enough progress, i.e. for which the condition in Line 19 of Algorithm~\ref{arls} is \emph{false},
and those that don't, i.e. for which the condition in Line 19 is \emph{true}.
We call the former \emph{Type 1} iterations and the latter \emph{Type 2} iterations.
Intuitively, Type 1 iterations signify that $g(x)$ is decreasing at a sufficient rate to achieve the desired
convergence, while Type 2 iterations indicate a local minimum that should be dealt with.
Our argument consists of two steps: Showing that as long as there are enough Type 1 iterations, a desired solution
will be obtained (Lemma~\ref{recurrence}), and bounding the total number of Type 2 iterations with constant probability (Lemma~\ref{type2_lemma}).

\begin{lemma}[Convergence rate]
If Algorithm~\ref{arls} executes at least $T_1 = s\log\frac{g(x^0)-f(x^*)}{\epsilon}$ Type 1 iterations, then
$f(x^T) \leq f(x^*) + \epsilon$.
\label{recurrence}
\end{lemma}
The proof of this lemma can be found in Appendix~\ref{proof_lemma_recurrence}.

\begin{lemma}[Bounding Type 2 iterations]
If $s \geq s^* \max\{4\widetilde{\kappa} + 7, 12\widetilde{\kappa} + 6\}$
and $\mathrm{opt} \geq f(x^*)$, then with probability at least $0.2$
the number of Type 2 iterations is at most $(s^*-1)(4\t{\kappa}+6)$.
\label{type2_lemma}
\end{lemma}
The proof of this lemma appears in Section~\ref{sec:type2_lemma}. 
These lemmas can now be
directly used to obtain the following lemma, which states the performance guarantee of the ARHT core routine (Algorithm~\ref{arls}).

\begin{lemma}[Algorithm~\ref{arls} guarantee]
If $s \geq s^* \max\{4\widetilde{\kappa} + 7, 12\widetilde{\kappa} + 6\}$
and $\mathrm{opt} \geq f(x^*)$, 
with probability at least $0.2$
$\mathrm{ARHT\_core}(s,\mathrm{opt},\epsilon)$ returns
an $s$-sparse solution $x$ such that $f(x) \leq \mathrm{opt} + \epsilon$.
\label{reg_prob_lemma}
\end{lemma}
\begin{proof}
By Lemma~\ref{type2_lemma}, with probability at least $0.2$
there will be at most $(s^*-1)(4\t{\kappa} + 6)$ Type 2 iterations.
This means that the number of Type 1 iterations is at least
\[ T - (s^*-1)(4\t{\kappa}+6) \geq 
s \log \frac{f(\vec{0})-B}{\epsilon}
\geq 
s \log \frac{g^0(x^0)-f(x^*)}{\epsilon} \]
where $B = \underset{x}{\min}\, f(x)$, and the latter inequality follows from the fact that $f(\vec{0}) = g^0(\vec{0}) \geq g^0(x^0)$
and $f(x^*) \geq B$.
Lemma~\ref{recurrence} then implies that 
$f(x^T) \leq f(x^*) + \epsilon$.
\end{proof}
In other words, as long as $\mathrm{opt}\geq f(x^*)$, a solution of value $\leq \mathrm{opt}+\epsilon$ will be found. As the value $\mathrm{opt}$ is not known a priori, we perform binary search on it, as described in Algorithm~\ref{local_reg}. Furthermore, the probability of success in the previous lemma can be boosted by repeating multiple times.
Combining these arguments will lead us to the proof of Theorem~\ref{reg_theorem}.
First, we turn the result of Lemma~\ref{reg_prob_lemma} into a high probability result by 
repeating multiple times:
\begin{lemma}
If $s \geq s^* \max\{4\widetilde{\kappa} + 7, 12\widetilde{\kappa} + 6\}$
and $\mathrm{opt} \geq f(x^*)$, 
$\mathrm{ARHT\_robust}(s, \mathrm{opt}, \epsilon, B)$ returns
an $s$-sparse solution $x$ such that $f(x) \leq \mathrm{opt} + \epsilon$
with probability at least $1-\frac{1}{6n\log \frac{f(\vec{0}) - B}{\epsilon}}$.
\label{reg_phase_lemma}
\end{lemma}
\begin{proof}
From Lemma~\ref{reg_prob_lemma}, the probability that a given call to $\mathrm{ARHT\_core}$ fails is at most $0.8$.
Since this random experiment is executed $5\log\left(6n\log \frac{f(\vec{0})-B}{\epsilon}\right)$ times independently, the probability that it never succeeds is
at most $\left(0.8\right)^{5\log\left(6n\log \frac{f(\vec{0})-B}{\epsilon}\right)} < \frac{1}{6n\log \frac{f(\vec{0})-B}{\epsilon}}$, therefore the statement follows.
\end{proof}

\begin{lemma}
If $s \geq s^* \max\{4\t{\kappa} + 7, 12\t{\kappa} + 6\}$,
$\mathrm{ARHT}(s,\epsilon)$ (in Algorithm~\ref{local_reg}) 
returns an $s$-sparse solution $x$
such that $f(x) \leq f(x^*) + \epsilon$.
The algorithm succeeds with probability at least $1 - \frac{1}{n}$.
and the number of calls to $\mathrm{ARHT\_robust}$ is $\leq 6\log \frac{f(\vec{0})-B}{\epsilon}$.
\label{reg_full_lemma}
\end{lemma}
\begin{proof}
First we will bound the number of calls to $\mathrm{ARHT\_robust}$. Let $L_k$ be the equal to $r-l$ before the $k$-th iteration in Line 21 of Algorithm~\ref{local_reg}.
Then either $L_{k+1} = L_k/2$ (Line 25) or $L_{k+1} \leq L_k/2 + \epsilon/3 < 5L_k/6$ (Line 28). Therefore in any case we have
$L_{k+1} < 5L_k/6$ which implies that after $T=6\log \frac{f(\vec{0}) - B}{\epsilon}$ iterations we will have $r-l \leq \epsilon$.

Now let us compute the probability that all the calls to $\mathrm{ARHT\_robust}$ are successful.
The number of such calls is at most $6\log \frac{f(\vec{0})-B}{\epsilon}$ and we know each one of them independently fails with probability less than
$\frac{1}{6n\log \frac{f(\vec{0})-f(B)}{\epsilon}}$, so by a union bound the probability that at least one call fails is less than
$\frac{1}{n}$.

To prove correctness, note that by Lemma~\ref{reg_phase_lemma}, for each $r \geq f(x^*)$ we have
$f(\mathrm{ARHT\_robust}(s,r,\epsilon/3,B)) \leq r + \epsilon/3$. After Line 20 of Algorithm~\ref{local_reg}, we will have
$l = B \leq f(x^*)$. In the while construct, it is always true that $f(x^*) \geq l$. This is initially
true, as we saw. For each $m$ chosen in Line 22 and $x$ in Line 23, note that if $f(x) > m + \epsilon/3$, then
by Lemma~\ref{reg_phase_lemma} $f(x^*) > m$ and so the invariant that $f(x^*) \geq l$ stays true.
On the other hand, it is always true that $f(b) \leq r$. Initially this is so because $f(\vec{0}) = r$, and
when we decrease $r$ to some $f(x)$ we also update $b=x$.
This implies that in the end of the algorithm the returned solution will have the required property, since
we will have $f(b) \leq r \leq l + \epsilon \leq f(x^*) + \epsilon$.
\end{proof}

The proof Theorem~\ref{reg_theorem} now easily follows.

\begin{prevproof}{Theorem}{reg_theorem}
Lemma~\ref{reg_full_lemma} already establishes the correctness of the algorithm with probability at least $1 - \frac{1}{n}$. For the runtime, note that
$\mathrm{ARHT\_core}$ takes $O\left(s\log \frac{f(\vec{0})-B}{\epsilon}\right)$ iterations,
$\mathrm{ARHT\_robust}$ takes $O\left(\log \left(n\log \frac{f(\vec{0})-B}{\epsilon}\right)\right)$ iterations,
and $\mathrm{ARHT}$ takes $O\left(\log \frac{f(\vec{0})-B}{\epsilon}\right)$ iterations.
In conclusion, the total number of iterations is 
$O\left(s\log^2 \frac{f(\vec{0})-B}{\epsilon}
	\log \left(n\log \frac{f(\vec{0})-B}{\epsilon}\right)
	\right)$,
	each of which
requires a constant number of minimizations of $f$.
\end{prevproof}

\subsection{Bounding Type 2 Iterations}
\label{sec:type2_lemma}

When $x$ has significant $\ell_2^2$ mass in the target support, the regularization term 
$\frac{\rho_2^+}{2} \left\Vert x\right\Vert_2^2$ 
might penalize the target solution too much, leading to a Type 2 iteration.
In this case, we use random sampling to detect an element in the optimal support and unregularize it.
This procedure escapes all local minima, thus leading to a bound in the total number of Type 2 iterations.

More concretely, we show that if at some iteration of the algorithm
the value of $g(x)$ does not decrease sufficiently (Type 2 iteration), then roughly 
at least a $\frac{1}{\widetilde{\kappa}}$-fraction of 
the $\ell_2^2$ mass of
$x$ lies in the target support $S^*$. 
We exploit this property by sampling an element $i$ proportional to $x_i^2$
and removing its corresponding term from the regularizer (\emph{unregularizing} it).
We show that with constant probability this will happen at most $O(s^*\widetilde{\kappa})$ times, as after that
all the elements in $S^*$ will have been unregularized.

When referring to the $t$-th iteration of Algorithm~\ref{arls}, we let $x^t$ be the current solution with support set $S^t$
and $R^t\subseteq[n]$ the current regularization set as defined in the algorithm.
For ease
of notation, we will drop the subscript of the regularizer, i.e. $\Phi^t(z) := \frac{\rho_{2}^+}{2} \left\Vert z_{R^t}\right\Vert_2^2$
and of the regularized function, i.e. $g^t(z) := f(z) + \Phi^t(z)$.
Note that by definition of the algorithm $x^t$ is 
an $S^t$-restricted minimizer of $g^t$.

Let $(\rho_{2}^+)'$ and $(\rho_{s+s^*}^-)'$ be RSS and RSC parameters of $g^t$.
We start with a lemma that relates $(\rho_{2}^+)'$ to $\rho_2^+$ and $(\rho_{s+s^*}^-)'$ to $\rho_{s+s^*}^-$, and
is proved in Appendix~\ref{proof_lemma_reg_condition}.
\begin{lemma}[RSC, RSS of regularized function]
$(\rho_{2}^+)'\leq 2 \rho_2^+$ and 
$(\rho_{s+s^*}^-)' \geq \rho_{s+s^*}^-$
\label{reg_condition}
\end{lemma}
This states that the restricted smoothness and strong convexity constants of the regularized function are always within a constant factor
of those of the original function, and thus we can make our statements in terms of the RSC, RSS of the original function.
Next, we present a lemma that establishes a lower bound on the progress $g^t(x^t) - g^{t+1}(x^{t+1})$ in one iteration. This will be helpful
in order to diagnose the cause of having insufficient progress in one iteration.

\begin{lemma}[ARHT Progress Lemma]
If $\mathrm{opt} \geq f(x^*)$,
for the progress $g^t(x^t) - g^t(x^{t+1})$
in Line 19 of Algorithm~\ref{arls} it holds that
\begin{align*}
& g^t(x^t) - g^t(x^{t+1})\\
& \geq \frac{\rho^-}{2|S^*\backslash S^t|\rho^+} \Big( f(x^t) - f(x^*)
+ \langle \nabla_{S^t\backslash S^*} \Phi^t(x^t), x_{S^t\backslash S^*}^t\rangle - \frac{1}{2\rho^-}\left\Vert\nabla_{S^t\cap S^*}\Phi^t(x^t)\right\Vert_2^2 
\Big)
- \rho^+ (x_j^t)^2
\end{align*}
\label{progress_lemma}
\end{lemma}
\begin{proof}
First of all, 
since the condition in Line 12 (``if $\underset{\mathrm{supp}(x)\subseteq S^t}{\min}\ f(x) \leq \mathrm{opt}$'')
was not triggered, we have that $\underset{\mathrm{supp}(x)\subseteq S^t}{\min}\ f(x) > \mathrm{opt}\geq f(x^*)$ and so 
$S^*\backslash S^t\neq \emptyset$.
By Lemma~\ref{reg_condition} we have that 
$(\rho^+)' \leq 2 \rho^+$, therefore the decrease in $g^t$ that is achieved
is
\begin{align*} 
& g^t(x^t) - g^t(x^{t+1})\\
& \geq \underset{\eta\in\mathbb{R}}{\max}\ \left\{g^t(x^t) -g^t(x^t + \eta \vec{1}_i - x_j^t \vec{1}_j)\right\}\\
& \geq \underset{\eta\in\mathbb{R}}{\max}\ \left\{-\langle \nabla g^t(x^t), \eta \vec{1}_i - x_j^t \vec{1}_j\rangle - \rho^+\eta^2 - \rho^+ (x_j^t)^2\right\} := B
\end{align*}
Note that, as defined by the algorithm, $x^t$ is an $S^t$-restricted minimizer of $g^t$ and since $j\in S^t$, we have $\nabla_j g^t(x^t) = 0$. Therefore
\begin{equation}
\label{B_ineq}
\begin{aligned}
B = & \max_{\eta\in\mathbb{R}} \{ -\langle \nabla g^t(x^t), \eta \vec{1}_i \rangle - \rho^+\eta^2 - \rho^+ (x_j^t)^2 \} \\
= & \frac{\left[\nabla_i g^t(x^t)\right]^2}{4\rho^+} - \rho^+ (x_j^t)^2\\
\geq & \underset{k\in S^*\backslash S}{\max}\ \frac{\left[\nabla_k g^t(x^t)\right]^2}{4\rho^+} - \rho^+ (x_j^t)^2\\
\geq & \frac{\left\Vert\nabla_{S^*\backslash S^t} g^t(x^t)\right\Vert_2^2}{4|S^*\backslash S^t|\rho^+} - \rho^+ (x_j^t)^2
\end{aligned}
\end{equation}
where we used the fact that $i$ was picked to maximize $\left|\nabla_k g^t(x^t)\right|$.
Now we would like to relate this to $g^t(x^t) - f(x^*)$ (and not $g^t(x^t) - g^t(x^*)$). By applying the Restricted Strong Convexity property,
\begin{align*}
&f(x^*) - f(x^t) \\
& \geq \langle \nabla f(x^t), x^* - x^t\rangle + \frac{\rho^-}{2} \left\Vert x^t-x^*\right\Vert_2^2\\
& \geq \langle \nabla f(x^t), x^* - x^t\rangle + \frac{\rho^-}{2} \left\Vert x^*_{S^*\backslash S^t}\right\Vert_2^2 + \frac{\rho^-}{2} \left\Vert (x^t-x^*)_{S^t\cap S^*}\right\Vert_2^2
\end{align*}
Now note that $f(x^t) = g^t(x^t) - \Phi^t(x^t)$, 
	$\nabla_{S^t} g^t(x^t) = \vec{0}$ 
	(since $x^t$ is an $S^t$-restricted minimizer of $g^t$), 
	and $\nabla\Phi^t(x^t) = \nabla_{S^t}\Phi^t(x^t)$
	therefore
\begin{align*}
\langle\nabla f(x^t), x^* - x^t\rangle & =  \langle\nabla g^t(x^t), x^* - x^t\rangle - \langle\nabla \Phi^t(x^t), x^* - x^t\rangle \\
& = \langle\nabla g_{S^*\backslash S^t}^t(x^t), x_{S^*\backslash S^t}^*\rangle 
+ \langle\nabla_{S^t\backslash S^*} \Phi^t(x^t), x_{S^t\backslash S^*}^t\rangle + \langle\nabla_{S^t\cap S^*} \Phi^t(x^t), (x^t-x^*)_{S^t\cap S^*}\rangle 
\end{align*}
Plugging this into the previous inequality, we get
\begin{align*}
& f(x^*) - f(x^t) \\
& \geq \langle\nabla g_{S^*\backslash S^t}^t(x^t), x_{S^*\backslash S^t}^*\rangle 
+ \frac{\rho^-}{2} \left\Vert x^*_{S^*\backslash S^t}\right\Vert_2^2 
+ \langle \nabla_{S^t\backslash S^*} \Phi^t(x^t), x_{S^t\backslash S^*}^t \rangle\\
& + \langle\nabla_{S^t\cap S^*} \Phi^t(x^t), (x^t-x^*)_{S^t\cap S^*}\rangle 
 + \frac{\rho^-}{2} \left\Vert (x^t-x^*)_{S^t\cap S^*}\right\Vert_2^2\\
& \geq 
-\frac{1}{2\rho^-} \left\Vert\nabla_{S^*\backslash S^t} g^t(x^t)\right\Vert_2^2 
 + \langle\nabla_{S^t\backslash S^*}\Phi^t(x^t), x_{S^t\backslash S^*}^t\rangle 
 - \frac{1}{2\rho^-}\left\Vert\nabla_{S^t\cap S^*}\Phi^t(x^t)\right\Vert_2^2
\end{align*}
where we twice used the inequality $\langle u, v\rangle 
+ \frac{\lambda}{2}\left\Vert v\right\Vert_2^2
\geq -\frac{1}{2\lambda} \left\Vert u\right\Vert_2^2$
for any $\lambda > 0$. This inequality is derived by expanding $\frac{1}{2} \left\Vert \frac{1}{\sqrt{\lambda}} u + \sqrt{\lambda} v\right\Vert_2^2\geq 0$.
So plugging in $\left\Vert\nabla_{S^*\backslash S^t} g^t(x^t)\right\Vert_2^2$ into (\ref{B_ineq}),
\begin{align*}
 B 
&\geq \frac{\rho^-}{2|S^*\backslash S^t|\rho^+} \Big( f(x^t) - f(x^*)
+ \langle \nabla_{S^t\backslash S^*} \Phi^t(x^t), x_{S^t\backslash S^*}^t\rangle - \frac{1}{2\rho^-}\left\Vert\nabla_{S^t\cap S^*}\Phi^t(x^t)\right\Vert_2^2 
\Big)
- \rho^+ (x_j^t)^2
\end{align*}
\end{proof}

Let $R\subseteq [n]$ be the set of currently regularized elements. The following invariant is a crucial ingredient for bringing the sparsity 
from $O(s^*\widetilde{\kappa}^2)$ down to $O(s^*\widetilde{\kappa})$, and we intend to enforce it
at all times.
It essentially
states that there will always be enough elements in the current solution
that are being regularized.
\begin{invariant}
\[ |R\cap S| \geq s^* \max\{1,8\widetilde{\kappa}\} \]
\label{central_inv}
\end{invariant}

To give some intuition on this, ARHT owes its improved $\widetilde{\kappa}$ dependence on the regularizer $\frac{\rho^+}{2} \left\Vert x\right\Vert_2^2$.
However, during the algorithm, some elements are being unregularized. 
Our analysis requires that the current solution support always contains $\Omega\left(s^* \widetilde{\kappa}\right)$ regularized elements, which is what Invariant~\ref{central_inv} states.

We can now proceed to show that, with constant probability, Algorithm~\ref{arls} will only have $O(s^*\t{\kappa})$ Type 2 iterations, which is the goal of this section.

\begin{prevproof}{Lemma}{type2_lemma}
We first observe some useful properties of our regularizer, which can be verified by simple substitution.
The definition of $\Phi^t(x^t)$ implies that
\begin{align}
\langle \nabla_{S^t\backslash S^*} \Phi^t(x^t), x_{S^t\backslash S^*}^t\rangle 
= \rho^+ \langle x_{R^t\backslash S^*}^t, x_{S^t\backslash S^*}^t\rangle
= \rho^+ \left\Vert x_{R^t\backslash S^*}^t\right\Vert_2^2 
\label{obs1}
\end{align}
and
\begin{align}
\left\Vert\nabla_{S^t\cap S^*} \Phi^t(x^t)\right\Vert_2^2 = (\rho^+)^2 \left\Vert x_{R^t\cap S^*}^t\right\Vert_2^2 
\label{obs2}
\end{align}
which also means that 
\begin{align}
\Phi^t(x^t) = \frac{1}{2} \langle \nabla_{S^t\backslash S^*} \Phi^t(x^t), x_{S^t\backslash S^*}^t\rangle + \frac{1}{2\rho^+}\left\Vert\nabla_{S^t\cap S^*} \Phi^t(x^t)\right\Vert_2^2
\label{obs3}
\end{align}
(\ref{obs1}),(\ref{obs2}), and (\ref{obs3}) will be used later on.
Now, before the first iteration we have $\left|R^0\cap S^0\right| = \left|S^0\right| = s$.
Since in each Type 2 iteration we have $R^{t+1} = R^t - 1$, 
\begin{align*}
\left|R^t\cap S^t\right| \geq s - \text{[number of Type 2 iterations up to $t$]}
\end{align*}
This implies that for the first $(s^*-1)(4\t{\kappa}+6)$ Type 2 iterations,
\begin{align}
|R^t\cap S^t| \geq s - (s^*-1)(4\t{\kappa}+6) \geq s^* \max\{1,8\t{\kappa}\}
\end{align}
since $s \geq s^*\max\left\{4\t{\kappa}+7,12\t{\kappa}+6\right\}$.
From this 
it follows that
\begin{align*}
 |(R^t\cap S^t)\backslash S^*| 
&= |R^t\cap S^t| - |R^t\cap S^t\cap S^*| \\
&\geq s^*\max\{1,8\t{\kappa}\} - |S^t\cap S^*| \\
&\geq |S^*\backslash S^t|8\t{\kappa}\\
&= |S^*\backslash S^t|8\frac{\rho^+}{\rho^-}
\end{align*}
and so
\begin{align*}
(x_j^t)^2 
&\leq \frac{1}{|(R^t\cap S^t)\backslash S^*|} \left\Vert x_{(R^t\cap S^t)\backslash S^*}^t \right\Vert_2^2 \\
&\leq \frac{\rho^-}{8|S^*\backslash S^t|\rho^+} \left\Vert x_{R^t\backslash S^*}^t \right\Vert_2^2\\
&=
\frac{\rho^-}{8|S^*\backslash S^t|(\rho^+)^2} \langle\nabla_{S^t\backslash S^*} \Phi^t(x^t), x_{S^t\backslash S^*}^t\rangle 
\end{align*}
where $j\in S^t$ is the element that the algorithm removes from $S^t$,
and we used (\ref{obs1}).
Combining this inequality with the statement of Lemma~\ref{progress_lemma} we have
\begin{equation}
\begin{aligned}
& g^t(x^t) - g^t(x^{t+1})\\
& \geq \frac{\rho^-}{2|S^*\backslash S^t|\rho^+} \Big( f(x^t) - f(x^*)
+ \langle \nabla_{S^t\backslash S^*} \Phi^t(x^t), x_{S^t\backslash S^*}^t\rangle 
- \frac{1}{2\rho^-}\left\Vert\nabla_{S^t\cap S^*}\Phi^t(x^t)\right\Vert_2^2 
\Big)
- \rho^+ (x_j^t)^2\\
&\geq \frac{\rho^-}{2|S^*\backslash S^t|\rho^+} \Big( f(x^t) - f(x^*)
+ \frac{3}{4}\langle \nabla_{S^t\backslash S^*} \Phi^t(x^t), x_{S^t\backslash S^*}^t\rangle 
- \frac{1}{2\rho^-}\left\Vert\nabla_{S^t\cap S^*}\Phi^t(x^t)\right\Vert_2^2 
\Big)
\end{aligned}
\label{eq_1}
\end{equation}
By definition of a Type 2 iteration,
\begin{equation}
\begin{aligned}
 g^t(x^t) - g^t(x^{t+1}) 
& < \frac{1}{s} \left(g^t(x^t) - \mathrm{opt}\right)\\
&\leq \frac{\rho^-}{2|S^*\backslash S^t| \rho^+} \left(g^t(x^t) - f(x^*)\right)\\
&= \frac{\rho^-}{2|S^*\backslash S^t| \rho^+} \left(f(x^t) - f(x^*) + \Phi^t(x^t)\right)
\end{aligned}
\label{eq_2}
\end{equation}
where we used the fact that $s \geq 2s^*\t{\kappa}\geq 2|S^*\backslash S^t|\t{\kappa}$ and $f(x^*) \leq \mathrm{opt}$.
Combining inequalities (\ref{eq_1}) and (\ref{eq_2}) we get
\begin{align*}
& \Phi^t(x^t) > \frac{3}{4} \langle\nabla_{S^t\backslash S^*}\Phi^t(x^t), x_{S^t\backslash S^*}^t\rangle - \frac{1}{2\rho^-} \left\Vert\nabla_{S^t\cap S^*}\Phi^t(x^t)\right\Vert_2^2 
\end{align*}
or equivalently, by replacing $\Phi^t(x^t)$ from (\ref{obs3}),
\begin{align*}
& \frac{1}{2}\left(\frac{1}{\rho^-} + \frac{1}{\rho^+}\right) \left\Vert\nabla_{S^t\cap S^*}\Phi^t(x^t)\right\Vert_2^2 >
\frac{1}{4} \langle\nabla_{S^t\backslash S^*}\Phi(x^t), x_{S^t\backslash S^*}^t\rangle
\end{align*}
Further applying (\ref{obs1}) and (\ref{obs2}), we equivalently get
\begin{align}
2\left(1 + \t{\kappa}\right) \left\Vert x_{R^t\cap S^*}^t\right\Vert_2^2 
>
\left\Vert x_{R^t\backslash S^*}^t\right\Vert_2^2
\label{l2_mass}
\end{align}
Now, note that in Lines 21-22 the algorithm
picks an element $i\in R^t$ with probability proportional to $(x_i^t)^2$ and unregularizes it, i.e. sets $R^{t+1}\leftarrow R^t\backslash\{i\}$.
We denote this probability distribution over $i\in R^t$ by $\mathcal{D}$.
From what we have established already in (\ref{l2_mass}), we can lower bound the probability that $i$ lies in the target support:
\begin{equation}
\begin{aligned}
\underset{i\sim\mathcal{D}} \Pr[i\in S^*] 
& = \frac{\left\Vert x_{R^t\cap S^*}^t\right\Vert_2^2}{\left\Vert x_{R^t\cap S^*}^t\right\Vert_2^2+\left\Vert x_{R^t\backslash S^*}^t\right\Vert_2^2} \\
& > \frac{\frac{1}{2(1+\t{\kappa})}}{1 + \frac{1}{2(1+\t{\kappa})}} \\
& = \frac{1}{2\t{\kappa}+3} \\
& := p
\end{aligned}
\label{probability}
\end{equation}
Note that this event can happen at most once for each $i\in S^*$ during the whole execution of the algorithm, since each element can only
be removed once from the set of regularized elements.

We will prove that 
with constant probability 
the number of Type 2 steps will be at most $(s^*-1)(4\t{\kappa}+6):= b$.
For $1\leq k \leq b$, we define the following random variables:
\begin{itemize}
\item $i_k\in[n]$ is the index picked in the $k$-th Type 2 iteration, or $\perp$ if there are less than $k$ Type 2 iterations.
\item {$q_k$ is the probability of picking an index in the optimal support in the $k$-th Type 2 iteration (i.e. $i_k\in S^*$):
\[ q_k =\begin{cases} 
	\left\Vert x_{R^{t_k}\cap S^*}^{t_k}\right\Vert_2^2 / \left\Vert x_{R^{t_k}}^{t_k}\right\Vert_2^2 & \text{if $i_k\neq\perp$}\\
0 & \text{otherwise}
\end{cases} \]
where $t_k\in[T]$ is the index of the $k$-th Type 2 iteration within all iterations of the algorithm.
Note that, by (\ref{probability}), $q_k > 0$ implies $q_k \geq p$.
}
\item {$X_k$ is $1$ if the index picked in the $k$-th Type 2 step was in the optimal support:
\begin{align*}
	X_k = \begin{cases}
1 & \text{with probability $q_k$}\\
0 & \text{otherwise}
\end{cases}
\end{align*}
}
\end{itemize}
Our goal is to upper bound $\Pr\left[\sum\limits_{k=1}^b X_k \leq s^*-1\right]$.
This automatically implies the same upper bound 
on the probability that there will be more than $b$ Type 2 iterations.

We define another sequence of random variables $Y_0, \dots, Y_b$, where $Y_0 = 0$, and 
\begin{align*}
Y_k = 
\begin{cases}
Y_{k-1} + \frac{p}{q_k} - p & \text{if $X_k=1$}\\
Y_{k-1} - p & \text{if $X_k=0$}
\end{cases}
\end{align*}
for $k\in[b]$.
Since if $q_k > 0$ we have $\frac{p}{q_k} \leq 1$, it is immediate that
\[ Y_k - Y_{k-1} \leq X_k - p \]
and so
$Y_b \leq \sum\limits_{k=1}^b X_k - bp$. %
Furthermore, 
\begin{align*}
\mathbb{E}\left[Y_k\ |\ i_1,\dots,i_{k-1}\right] 
& = Y_{k-1} + q_k\left(\frac{p}{q_k} - p\right) - \left(1-q_k\right) p \\
&= Y_{k-1} 
\end{align*}
meaning that $Y_0,\dots,Y_b$ is a martingale with respect to $i_1,\dots,i_b$. %
We will apply the inequality from Lemma~\ref{martingale_concentration}. We compute a bound
on the differences
\begin{align}
Y_{k-1} - Y_k &=\begin{cases}
p - \frac{p}{q_k} & \text{if $X_k=1$}\\
p & \text{if $X_k=0$}
\end{cases} \label{expectation_diff}\\
&\leq p \notag
\end{align}
and the variance
\begin{align*}
\mathrm{Var}\left(Y_k\ |\ i_1,\dots,i_{k-1}\right) &= \mathbb{E}\left[\left(Y_k - \mathbb{E}\left[Y_k\ |\ i_1,\dots,i_{k-1}\right]\right)^2\ |\ i_1,\dots,i_{k-1}\right] \\
&=\mathbb{E}\left[\left(Y_{k} - Y_{k-1}\right)^2\ |\ i_1,\dots,i_{k-1}\right] \\
&= q_k\cdot \left(p - \frac{p}{q_k}\right)^2 + \left(1-q_k\right)\cdot p^2\\
&= q_k\cdot p^2\left(1 - \frac{2}{q_k} + \frac{1}{q_k^2}\right) + \left(1-q_k\right)\cdot p^2\\
&= p^2\left(\frac{1}{q_k}-1\right)\\
&\leq p
\end{align*}
where we used (\ref{expectation_diff}) along with the fact that $q_k \geq p$.
Using the concentration inequality from 
Lemma~\ref{martingale_concentration} we obtain
\begin{align*}
\Pr\left[\sum\limits_{k=1}^b X_k \leq s^* - 1\right] 
&\leq \Pr\left[Y_b \leq s^* - 1 - b\cdot p\right] \\
&\leq e^{-(bp-s^*+1)^2/(2\left(b\cdot p + p\cdot \left(bp - s^* + 1\right)/3\right))}\\
&= e^{-(s^*-1)/(2(2+p/3))}\\
&\leq e^{-1/(2(2+1/9))}\\
&< 0.8 
\end{align*}
where we used the fact 
that $b p = 2 (s^* - 1)$,
$s^* \geq 2$ (otherwise the problem is trivial),
and 
$p = \frac{1}{2\widetilde{\kappa}+3} \leq \frac{1}{3}$. 
Therefore 
we conclude that the probability that we have not unregularized
the whole set $S^*$ after $b$ steps is at most $0.8$.
Since we can only have a Type 2 step if there is a regularized element in $S^*$ (this is immediate e.g. from (\ref{probability})), this implies that
with probability at least $0.2$ the number of Type 2 steps is at most $b=(s^*-1)(4\t{\kappa}+6)$.

\end{prevproof}

\subsection{Corollaries}
As the first corollary of Theorem~\ref{reg_theorem}, we show that it directly implies solution recovery bounds
similar to those of \cite{Zhang11}, 
while also improving the recovery bound by a constant factor.
\begin{corollary}[Solution recovery]
Given a function $f$ and an (unknown) $s^*$-sparse solution $x^*$,
such that the Restricted Gradient Optimal Constant at sparsity level $s$ is $\zeta$, i.e.
\[ \left|\langle\nabla f(x^*), y\rangle\right| \leq \zeta \left\Vert y\right\Vert_2 \]
for all $s$-sparse $y$ and
as long as 
\[ s \geq s^*\max \left\{4\widetilde{\kappa} + 7, 12\widetilde{\kappa} + 6\right\} \]
Algorithm~\ref{local_reg} ensures that
\[ f(x) \leq f(x^*) + \epsilon \] 
and
\begin{align*}
\left\Vert x - x^*\right\Vert_2 
& \leq \frac{\zeta}{\rho^-} \left(1 + \sqrt{1 + 2\epsilon\frac{\rho^-}{\zeta^2}}\right)\\
\end{align*}
For any $\theta > 0$ and $\epsilon \leq \frac{\zeta^2}{\rho^-}\theta(1 + \frac{\theta}{2})$,
this implies that 
\begin{align*}
\left\Vert x - x^*\right\Vert_2 
& \leq (2+\theta) \frac{\zeta}{\rho^-}
\end{align*}
\label{restricted_gradient_corollary}
\end{corollary}
\begin{proof}
By strong convexity we have
\begin{align*}
 \epsilon 
& \geq f(x) - f(x^*) \\
& \geq \langle \nabla f(x^*), x - x^*\rangle + \frac{\rho^-}{2} \left\Vert x-x^* \right\Vert_2^2\\
& \geq -\zeta \left\Vert x - x^*\right\Vert_2 + \frac{\rho^-}{2} \left\Vert x-x^*\right\Vert_2^2
\end{align*}
therefore
\begin{align*}
\frac{\rho^-}{2} \left\Vert x-x^*\right\Vert_2^2 - \zeta \left\Vert x-x^*\right\Vert_2 - \epsilon \leq 0
\end{align*}
looking at which as a quadratic polynomial in $\left\Vert x-x^*\right\Vert_2$, it follows that 
\begin{align*}
 \left\Vert x-x^*\right\Vert_2 &
 \leq \frac{\zeta + \sqrt{\zeta^2+2\epsilon\rho^-}}{\rho^-}\\
& = \frac{\zeta}{\rho^-} \left(1 + \sqrt{1 + 2\epsilon\frac{\rho^-}{\zeta^2}}\right)\\
& = (2+\theta) \frac{\zeta}{\rho^-}
\end{align*}
by setting $\epsilon = \frac{\zeta^2}{\rho^-} \left(\theta + \frac{1}{2} \theta^2\right)$.
\end{proof}

The next corollary shows that our Theorem~\ref{reg_theorem} can be also used to obtain
support recovery results under a ``Signal-to-Noise'' condition given as a lower bound to $|x_{\min}^*|$.

\begin{corollary}[Support recovery]
As long as
\[ s \geq s^*\max \left\{4\widetilde{\kappa} + 7, 12\widetilde{\kappa} + 6\right\} \]
and 
$|x_{\min}^*| > \frac{\zeta}{\rho^-}$,
Algorithm~\ref{local_reg} 
with $\epsilon < -\frac{1}{2\rho^-}\zeta^2 + \frac{\rho^-}{2} (x_{\min}^*)^2 $
returns a solution $x$ with support $S$ such that
\[ S^* \subseteq S\]
\label{support_recovery}
\end{corollary}
\begin{proof}
Let us suppose that $S^*\backslash S^t\neq \emptyset$.
By restricted strong convexity we have
\begin{align*}
 -\frac{1}{2\rho^-} \zeta^2 + \frac{\rho^-}{2}\left(x_{\min}^*\right)^2
& > \epsilon \\
& \geq f(x) - f(x^*) \\
&\geq \langle \nabla f(x^*), x-x^*\rangle + \frac{\rho^-}{2} \left\Vert x-x^*\right\Vert_2^2 \\
&\geq \langle \nabla f(x^*), x\rangle + \frac{\rho^-}{2}
		\left\Vert x_{S^t\backslash S^*}\right\Vert_2^2
	+ 	\frac{\rho^-}{2}\left\Vert x_{S^*\backslash S^t}^*\right\Vert_2^2\\ 
&\geq -\frac{1}{2\rho^-}\left\Vert \nabla_{S^t\backslash S^*} f(x^*)\right\Vert_2^2
	+ 	\frac{\rho^-}{2}\left\Vert x_{S^*\backslash S^t}^*\right\Vert_2^2\\
& \geq -\frac{1}{2\rho^-} \zeta^2 + \frac{\rho^-}{2}\left(x_{\min}^*\right)^2
\end{align*}
a contradiction.
Here we used the fact that by local optimality $\nabla_{S^*} f(x^*) = \vec{0}$,
the inequality $\langle u,v\rangle + \frac{\lambda}{2} \left\Vert v\right\Vert_2^2 \geq -\frac{1}{2\lambda}\left\Vert u\right\Vert_2^2$
for any vectors $u,v$ and scalar $\lambda > 0$,
and the fact that 
$\left\Vert \nabla_{S^t\backslash S^*} f(x^*)\right\Vert_2^2 \leq \zeta^2$ by Definition~\ref{RGOC}.
Therefore $S^*\subseteq S^t$.
\end{proof}

\section{Analysis of Orthogonal Matching Pursuit with Replacement (OMPR)}
\subsection{Overview and Main Theorem}
The OMPR algorithm was first described (under a different name) in~\cite{SSZ10}. It is an extension of OMP
but after each iteration some element is removed from $S^t$ so that the sparsity remains the same.
The algorithm description is in Algorithm~\ref{local}.

For each iteration $t$ of Algorithm~\ref{local}, we will define a solution
\[ \widetilde{x}^t = \underset{\mathrm{supp}(x)\subseteq S^t\cup S^*}{\mathrm{argmin}}\ f(x) \]
to be the optimal solution supported on $S^t\cup S^*$.
Furthermore, we let
$\overline{x}^{*}$
be the optimal $(s+s^*)$-sparse solution, i.e.
\[ \overline{x}^{*} = \underset{|\mathrm{supp}(x)|\leq s+s^*}{\mathrm{argmin}}\ f(x) \]
By definition, the following chain of inequalities holds
\begin{align*}
\underset{x\in\mathbb{R}^n}{\min}\ f(x) \leq f(\overline{x}^*) \leq f(\widetilde{x}^t) \leq \min\{f(x^t), f(x^*)\}
\end{align*}

We will assume that $s \leq 20 \t{\kappa} s^*$, as the other case is subsumed by Algorithm~\ref{local_reg}. Let us
also denote $\mu = \sqrt{\frac{s^*}{s}}$.

The following technical lemma is important for our approach, and roughly states that 
if there is significant $\ell_2$ norm difference between $x^t$ and $x^*$, at least one of $x^t,x^*$ is significantly larger
than $\widetilde{x}^t$ in function value.
Its importance lies on the fact that instead of directly applying strong convexity between $x^t$ and $x^*$, it gets a tighter bound
by making use of $\widetilde{x}^t$.
\begin{lemma}
For any function $f$ with RSC constant $\rho^-$ at sparsity level $s+s^*$ and any two solutions
$x^t$,$x^*$ with respective supports $S^t$, $S^*$ and 
sparsity levels $s$, $s^*$, we have that
\begin{align*}
\left(\sqrt{f(x^t)-f(\widetilde{x}^t)} + \sqrt{f(x^*)-f(\widetilde{x}^t)}\right)^2 
\geq \frac{\rho^-}{2} \left(\left\Vert x^*_{S^*\backslash S^t}\right\Vert_2^2 + \left\Vert x_{S^t\backslash S^*}^t\right\Vert_2^2\right)
\end{align*}
\label{lemma_strongconv}
\end{lemma}
The proof can be found in Appendix~\ref{proof_lemma_strongconv}.
The following theorem is the main result of this section. Its strength lies in its generality, and
various useful corollaries can be directly extracted from it.
\begin{theorem}
Given a function $f$, an (unknown) $s^*$-sparse solution $x^*$, a desired solution sparsity level $s$, and error
parameters $\epsilon > 0$ and 
$0 < \theta < 1$, 
Algorithm~\ref{local} returns an $s$-sparse solution
$x$ such that\\
\textbullet\ If $\widetilde{\kappa}\sqrt{\frac{s^*}{s}} \leq 1$, then
\begin{align*}
f(x) \leq f(x^*) + \epsilon
\end{align*}
in $O\left(\sqrt{ss^*}\log\frac{f(x^0)-f(x^*)}{\epsilon}\right)$ iterations.\\
\textbullet\ If $1 < \widetilde{\kappa}\sqrt{\frac{s^*}{s}} < 2 - \theta$, then
\begin{align*}
f(x) \leq f(x^*) + B
\end{align*}
where 
\begin{align*}
B = \epsilon + \frac{4(1-\theta)\left(\widetilde{\kappa}\sqrt{\frac{s^*}{s}}-1\right)}{\left(2-\widetilde{\kappa}\sqrt{\frac{s^*}{s}}-\theta\right)^2} (f(x^*) - f(\overline{x}^*))
\end{align*}
in $O\left(\frac{\sqrt{ss^*}}{\theta} \log \frac{f(x^0)-f(x^*)}
		{ B}
		\right)$
iterations.
\label{local_theorem}
\end{theorem}

\subsection{Progress Lemma and Theorem Proof}
The main ingredient needed to prove Theorem~\ref{local_theorem} is the following lemma,
which bounds the progress of Algorithm~\ref{local_reg} in one iteration. 
\begin{lemma}[OMPR Progress Lemma]
We can bound the progress of one step of the algorithm by distinguishing the following three cases:\\
\textbullet\ If $\mu\t{\kappa} \leq 1$, then
\begin{align*}
& f(x^{t+1}) - f(x^*) \leq \left(f(x^t) - f(x^*)\right) \left(1 - \frac{\mu}{|S^*\backslash S^t|}\right)
\end{align*}
\textbullet\ If $\mu\t{\kappa} > 1$ and $f(x^*) = f(\t{x}^t)$, then
\begin{align*}
f(x^{t+1}) - f(x^*) \leq \left(f(x^t) - f(x^*)\right) 
\cdot \left(1 - \frac{\mu}{|S^*\backslash S^t|}\left(2 - \mu\t{\kappa}\right)\right)
\end{align*}
\textbullet\ If $\mu\t{\kappa} > 1$ and $f(x^*) > f(\t{x}^t)$, then
\begin{align*}
f(x^{t+1}) - f(x^*) \leq \left(f(x^t) - f(x^*)\right)
\cdot\left(1 - \frac{\mu}{|S^*\backslash S^t|}\left(2 - \mu\t{\kappa} - \frac{2(\mu\t{\kappa}-1)}{\sqrt{\frac{f(x^t)-f(\widetilde{x}^t)}{f(x^*)-f(\widetilde{x}^t)}}-1}\right)\right)
\end{align*}
\label{exhaustive_technical}
\end{lemma}
\begin{proof}
First of all, if $S^*\subseteq S^t$ then, since $x^t$ is an $S^t$-restricted minimizer, we have
$f(x^T) \leq f(x^t) \leq f(x^*)$ and we are done. So suppose otherwise, i.e. $S^*\backslash S^t\neq \emptyset$ and $f(x^t) > f(x^*)$.
Let $i = \underset{i\notin S^t}{\mathrm{argmax}}\ \left|\nabla_i f(x^t)\right|$ and 
$j = \underset{j\in S^t}{\mathrm{argmin}}\ \left|x_j^t\right|$.
By definition of OMPR (Algorithm~\ref{local}) and restricted smoothness of $f$, we have 
\begin{equation}
\begin{aligned}
f(x^{t+1})
& \leq \underset{\eta\in\mathbb{R}}{\min}\ f(x^t +\eta\vec{1}_i- x_j^t \vec{1}_j)\\
    & \leq \underset{\eta\in\mathbb{R}}{\min}\ f(x^t) + \langle \nabla f(x^t), \eta \vec{1}_i -x_j^t\vec{1}_j\rangle + \frac{\rho^+}{2} \left\Vert \eta \vec{1}_i - x_j^t \vec{1}_j\right\Vert_2^2\\
    & = \underset{\eta\in\mathbb{R}}{\min}\ f(x^t) + \eta \nabla_i f(x^t) + \frac{\rho^+}{2} \eta^2 + \frac{\rho^+}{2} (x_j^t)^2 \\
    & = f(x^t) - \frac{\left(\nabla_i f(x^t)\right)^2}{2\rho^+} + \frac{\rho^+}{2} (x_j^t)^2 \\
    & \leq f(x^t) - \frac{\left\Vert\nabla_{S^*\backslash S^t} f(x^t)\right\Vert_2^2}{2\rho^+ |S^*\backslash S^t|} + \frac{\rho^+}{2|S^t\backslash S^*|} \left\Vert x_{S^t\backslash S^*}^t\right\Vert_2^2 
\end{aligned}	
\label{eq:smoothness_prog}
\end{equation}
where the second to last equality follows from the fact that $\nabla_j f(x^t) = \vec{0}$, as $x^t$ is an $S^t$-restricted minimizer of $f$,
and the last inequality since 
\[ (x_j^t)^2 = \underset{j\in S^t\backslash S^*}{\mathrm{min}}\ (x_j^t)^2 \leq \frac{\left\Vert x_{S^t\backslash S^*}^t\right\Vert_2^2}{|S^t\backslash S^*|} \]
Re-arranging (\ref{eq:smoothness_prog}), we get
\begin{align}
     |S^*\backslash S^t|(f(x^t) - f(x^{t+1}))
	\geq \frac{\left\Vert\nabla_{S^*\backslash S^t} f(x^t)\right\Vert_2^2}{2\rho^+}  - \frac{\rho^+}{2} \frac{|S^*\backslash S^t|}{|S^t\backslash S^*| } \left\Vert
	 x_{S^t\backslash S^*}^t\right\Vert_2^2 
	 \label{eq:progress_simple}
\end{align}
On the other hand, by restricted strong convexity of $f$,
\begin{equation}
\begin{aligned}
f(x^*) - f(x^t) 
& \geq \langle \nabla f(x^t), x^* - x^t\rangle + \frac{\rho^-}{2} \left\Vert x^*-x^t \right\Vert_2^2\\
& = \langle \nabla_{S^*\backslash S^t} f(x^t), x_{S^*\backslash S^t}^*\rangle + \frac{\rho^-}{2} \left\Vert x^*-x^t \right\Vert_2^2\\
& \geq \langle \nabla_{S^*\backslash S^t} f(x^t), x_{S^*\backslash S^t}^*\rangle 
+ \frac{\rho^-}{2} \left\Vert x_{S^*\backslash S^t}^* \right\Vert_2^2
+ \frac{\rho^-}{2} \left\Vert x_{S^t\backslash S^*}^t \right\Vert_2^2\\
& \geq \langle \nabla_{S^*\backslash S^t} f(x^t), x_{S^*\backslash S^t}^*\rangle
+ \frac{\mu\rho^+}{2} \left\Vert x_{S^*\backslash S^t}^* \right\Vert_2^2
+ \frac{\rho^- - \mu\rho^+}{2} \left\Vert x_{S^*\backslash S^t}^* \right\Vert_2^2
+ \frac{\rho^-}{2} \left\Vert x_{S^t\backslash S^*}^t \right\Vert_2^2\\
& \geq -\frac{1}{2\mu \rho^+} \left\Vert \nabla_{S^*\backslash S^t} f(x^t) \right\Vert_2^2
 + \frac{\rho^- - \mu\rho^+}{2} \left\Vert x_{S^*\backslash S^t}^* \right\Vert_2^2
+ \frac{\rho^-}{2} \left\Vert x_{S^t\backslash S^*}^t \right\Vert_2^2
\end{aligned}
\label{eq:strconv_prog}
\end{equation}
where the first equality follows from the fact that $\nabla_{S^t} f(x^t) = \vec{0}$ as $x^t$ is an $S^t$-restricted minimizer of $f$
and the last inequality from using the fact that
$\langle u, v\rangle 
+ \frac{\lambda}{2}\left\Vert v\right\Vert_2^2
\geq -\frac{1}{2\lambda} \left\Vert u\right\Vert_2^2$
for any $\lambda > 0$. 

Re-arranging (\ref{eq:strconv_prog}), we get
\begin{align*}
\frac{1}{2\mu\rho^+} \left\Vert \nabla_{S^*\backslash S^t} f(x^t) \right\Vert_2^2 
\geq  f(x^t) - f(x^*) 
- \frac{\mu \rho^+ - \rho^-}{2} \left\Vert x_{S^*\backslash S^t}^* \right\Vert_2^2
+ \frac{\rho^-}{2} \left\Vert x_{S^t\backslash S^*}^t \right\Vert_2^2
\end{align*}
By substituting this into (\ref{eq:progress_simple}),
\begin{align*}
     & |S^*\backslash S^t|(f(x^t) - f(x^{t+1}))\\
     & \geq \mu\left(f(x^t) - f(x^*) \right)
	- \frac{\mu^2 \rho^+ - \mu\rho^-}{2} \left\Vert x_{S^*\backslash S}^* \right\Vert_2^2
	+ \frac{\mu \rho^-}{2} \left\Vert x_{S^t\backslash S^*}^t \right\Vert_2^2
	- \frac{\rho^+}{2} \frac{|S^*\backslash S^t|}{|S^t\backslash S^*| } \left\Vert x_{S^t\backslash S^*}^t\right\Vert_2^2 
\end{align*}
Note that by our choice of $\mu$ and since $s^*\leq s$,
\[ \mu^2\rho^+ = \rho^+ \frac{s^*}{s} 
\geq 
\rho^+ \frac{s^* - |S^*\cap S^t|}{s - |S^*\cap S^t|} 
= \rho^+ \frac{|S^*\backslash S^t|}{|S^t\backslash S^*|} \]
and so
\begin{align*}
	& \mu\left(f(x^t) - f(x^*) \right)
	- \frac{\mu^2 \rho^+ - \mu\rho^-}{2} \left\Vert x_{S^*\backslash S}^* \right\Vert_2^2
	+ \frac{\mu \rho^-}{2} \left\Vert x_{S^t\backslash S^*}^t \right\Vert_2^2
	- \frac{\rho^+}{2} \frac{|S^*\backslash S^t|}{|S^t\backslash S^*| } \left\Vert x_{S^t\backslash S^*}^t\right\Vert_2^2 \\
	& \geq \mu (f(x^t) - f(x^*)) 
	- \frac{\mu}{2}\left(\mu \rho^+-\rho^-\right) \left(\left\Vert x_{S^*\backslash S^t}^* \right\Vert_2^2
	+ \left\Vert x_{S^t\backslash S^*}^t \right\Vert_2^2\right)
\end{align*}
concluding that
\begin{align*}
    |S^*\backslash S^t|(f(x^t) - f(x^{t+1}))
	\geq \mu (f(x^t) - f(x^*)) 
	- \frac{\mu}{2}\left(\mu \rho^+-\rho^-\right) \left(\left\Vert x_{S^*\backslash S^t}^* \right\Vert_2^2
	+ \left\Vert x_{S^t\backslash S^*}^t \right\Vert_2^2\right)
\end{align*}
For $\mu\t{\kappa} \leq 1 \Leftrightarrow \mu \rho^+-\rho^- \leq 0$, this automatically implies that
\begin{align*}
	& f(x^t) - f(x^{t+1}) \geq \frac{\mu}{|S^*\backslash S^t|} \left(f(x^t) - f(x^*)\right)\\
	\Leftrightarrow & f(x^{t+1})-f(x^*) \leq \left(1-\frac{\mu}{|S^*\backslash S^t|}\right) \left(f(x^t) - f(x^*)\right)
\end{align*}
On the other hand, if $\mu\t{\kappa} > 1$ we have
\begin{align*}
	|S^*\backslash S^t|\left(f(x^t) - f(x^{t+1})\right) 
	& \geq \mu (f(x^t) - f(x^*)) 
	-\frac{\mu}{2}\left(\mu \rho^+ - \rho^-\right) \left(\left\Vert x_{S^*\backslash S^t}^* \right\Vert_2^2
	+ \left\Vert x_{S^t\backslash S^*}^t \right\Vert_2^2\right) \\
	& \geq \mu (f(x^t) - f(x^*)) 
	- 
	\mu \left(\mu \t{\kappa} - 1\right) \left(\sqrt{f(x^t)-f(\widetilde{x}^t)} + \sqrt{f(x^*)-f(\widetilde{x}^t)}\right)^2
\end{align*}
where we used Lemma~\ref{lemma_strongconv}. 
If $f(x^*) = f(\widetilde{x}^t)$ it is immediate that
\begin{align*}
f(x^{t+1}) - f(x^*) \leq \left(1-\frac{\mu}{|S^*\backslash S^t|}\left(2-\mu\t{\kappa}\right)\right) \left(f(x^t)-f(x^*)\right) 
\end{align*}
so let us from now on assume that $f(x^*) > f(\widetilde{x}^t)$
and set $a = f(x^t) - f(\widetilde{x}^t)$,
$a' = f(x^{t+1}) - f(\widetilde{x}^t)$,
and $b = f(x^*) - f(\widetilde{x}^t)$. From what we have concluded before
\[ |S^*\backslash S^t|\left(a - a'\right) \geq \mu (a-b) - \mu(\mu\t{\kappa}-1)\left(\sqrt{a} + \sqrt{b}\right)^2\]
or equivalently
\begin{align*}
a' - b 
& \leq \left(1 - \frac{\mu}{|S^*\backslash S^t|}\right) (a-b) + \frac{\mu}{|S^*\backslash S^t|} \left(\mu\t{\kappa}-1\right) \left(\sqrt{a} + \sqrt{b}\right)^2\\
& = (a-b)\left(1 - \frac{\mu}{|S^*\backslash S^t|}\left(1 - (\mu\t{\kappa}-1) \frac{\left(\sqrt{a} + \sqrt{b}\right)^2}{a-b}\right)\right)\\
& = (a-b)\left(1 - \frac{\mu}{|S^*\backslash S^t|}\left(1 - (\mu\t{\kappa}-1) \frac{\sqrt{\frac{a}{b}} + 1}{\sqrt{\frac{a}{b}} - 1}\right)\right)\\
& = (a-b)\left(1 - \frac{\mu}{|S^*\backslash S^t|}\left(1 - (\mu\t{\kappa}-1) \left(1 + \frac{2}{\sqrt{\frac{a}{b}}-1}\right)\right)\right)\\
& = (a-b)\left(1 - \frac{\mu}{|S^*\backslash S^t|} \left(2 - \mu\t{\kappa} - \frac{2(\mu\t{\kappa}-1)}{\sqrt{\frac{a}{b}}-1}\right)\right)\\
\end{align*}
Replacing back $a,a',b$, the desired statement follows:
\begin{align*}
f(x^{t+1}) - f(x^*) \leq (f(x^t) - f(x^*))
\cdot \left(1 - \frac{\mu}{|S^*\backslash S^t|} \left(2 - \mu\t{\kappa} - \frac{2(\mu\t{\kappa}-1)}{\sqrt{\frac{f(x^t)-f(\widetilde{x}^t)}{f(x^*)-f(\widetilde{x}^t)}}-1}\right)\right)\\
\end{align*}
\end{proof}

The proof of Theorem~\ref{local_theorem} now follows by appropriately applying Lemma~\ref{exhaustive_technical}.

\begin{prevproof}{Theorem}{local_theorem}
\paragraph{Case 1: $\mu\t{\kappa}\leq 1$.}

By Lemma~\ref{exhaustive_technical}, we have 
\begin{align*}
f(x^{T}) - f(x^*) 
& \leq \left(f(x^{T-1}) - f(x^*)\right) \left(1 - \frac{\mu}{|S^*\backslash S^{T-1}|}\right)\\
& \leq \left(f(x^{T-1}) - f(x^*)\right) \left(1 - \frac{\mu}{s^*}\right)\\
& \leq \left(f(x^{T-1}) - f(x^*)\right) e^{- \frac{\mu}{s^*}}\\
& \leq \dots \\
& \leq \left(f(x^0) - f(x^*)\right) e^{- T\frac{\mu}{s^*}}\\
& \leq \epsilon
\end{align*}
for our choice of $T = O\left(\sqrt{s s^*} \log \frac{f(x^0) - f(x^*)}{\epsilon} \right)$ and replacing $\mu = \sqrt{\frac{s^*}{s}}$.
\paragraph{Case 2: $\mu\t{\kappa} > 1$.}
Let $\mathcal{A}$ be the set of $0\leq t \leq T-1$ such that $f(x^*) = f(\widetilde{x}^t)$ 
and $\mathcal{B}$ the 
set of $0\leq t \leq T-1$ such that $f(x^*) > f(\widetilde{x}^t)$.
By Lemma~\ref{exhaustive_technical}, for $t\in\mathcal{A}$ we then have
\begin{align*}
f(x^{t+1}) - f(x^*) 
& \leq \left(f(x^{t}) - f(x^*)\right) \left(1 - \frac{\mu}{|S^*\backslash S^t|}\left(2 - \mu\t{\kappa}\right)\right)\\
& \leq \left(f(x^{t}) - f(x^*)\right) \left(1 - \frac{\mu}{s^*}\left(2 - \mu\t{\kappa}\right)\right)
\end{align*}
We now consider the case $t\in \mathcal{B}$. 
By Lemma~\ref{exhaustive_technical},
\begin{equation}
\begin{aligned}
f(x^{t+1}) - f(x^*) \leq \left(f(x^t) - f(x^*)\right)
\cdot\left(1 - \frac{\mu}{|S^*\backslash S^t|}\left(2 - \mu\t{\kappa} - \frac{2(\mu\t{\kappa}-1)}{\sqrt{\frac{f(x^t)-f(\widetilde{x}^t)}{f(x^*)-f(\widetilde{x}^t)}}-1}\right)\right)
\end{aligned}
\label{helper_eq}
\end{equation}
Let us suppose that the theorem statement is not true. This implies
\begin{equation}
\begin{aligned}
 f(x^t) - f(x^*) 
& \geq f(x^T) - f(x^*)\\ 
& > \epsilon + \frac{4(1-\theta)(\mu\t{\kappa}-1)}{\left(2-\mu\t{\kappa}- \theta\right)^2} (f(x^*) - f(\overline{x}^*)) \\
& \geq \epsilon + \frac{4(1-\theta)(\mu\t{\kappa}-1)}{\left(2-\mu\t{\kappa}- \theta\right)^2} (f(x^*) - f(\widetilde{x}^t)) \\
& \geq \frac{4(1-\theta)(\mu\t{\kappa}-1)}{\left(2-\mu\t{\kappa}- \theta\right)^2} (f(x^*) - f(\widetilde{x}^t)) 
\label{exh_contra}
\end{aligned}
\end{equation}
for all $0\leq t \leq T$. Therefore 
\begin{align*}
f(x^t) - f(\widetilde{x}^t) 
& > \left(\frac{4(1-\theta)(\mu \t{\kappa}-1)}{\left(2-\mu\t{\kappa}-\theta\right)^2}+1\right) (f(x^*) - f(\widetilde{x}^t)) \\
& = \left(\frac{4(1-\theta)(\mu \t{\kappa}-1) + 4 + (\mu\t{\kappa}+\theta)^2 -4(\mu\t{\kappa} + \theta)}{\left(2-\mu\t{\kappa}-\theta\right)^2}\right)
\cdot (f(x^*) - f(\widetilde{x}^t)) \\
& = \frac{(\mu\t{\kappa} - \theta)^2}{\left(2-\mu\t{\kappa}-\theta\right)^2} (f(x^*) - f(\widetilde{x}^t)) 
\end{align*}
or equivalently for all $t\in \mathcal{B}$
\begin{align*}
\sqrt{\frac{f(x^t) - f(\widetilde{x}^t)}{f(x^*) - f(\widetilde{x}^t)}} - 1 > 
\frac{\mu\t{\kappa} - \theta}{2-\mu\t{\kappa}-\theta} -1  
=
\frac{2(\mu\t{\kappa} - 1)}{2-\mu\t{\kappa}-\theta}
\end{align*}
Replacing this into (\ref{helper_eq}), we get that for any $t\in \mathcal{B}$
\begin{align*}
 f(x^{t+1}) - f(x^*)
& \leq (f(x^t) - f(x^*))
\cdot\left(1 - \frac{\mu}{|S^*\backslash S^t|} \left(2 - \mu\t{\kappa} - \frac{2(\mu\t{\kappa}-1)}{\sqrt{\frac{f(x^t)-f(\widetilde{x}^t)}{f(x^*)-f(\widetilde{x}^t)}}-1}\right)\right)\\
& \leq (f(x^t) - f(x^*))\left(1 - \frac{\mu}{|S^*\backslash S^t|} \theta\right)
\end{align*}
and so combining it with the case $t\in \mathcal{A}$ and using the fact that $\mu\t{\kappa} < 2 - \theta\Leftrightarrow\theta< 2 - \mu\t{\kappa}$,
\begin{align*}
f(x^{T}) - f(x^*) 
& \leq \left(f(x^{T-1}) - f(x^*)\right) \left(1- \frac{\mu}{|S^*\backslash S^{T-1}|} \min\left\{2-\mu\t{\kappa}, \theta\right\}\right)\\
& \leq \left(f(x^{T-1}) - f(x^*)\right) \left(1- \frac{\mu}{s^*} \theta\right)\\
& \leq \left(f(x^{T-1}) - f(x^*)\right) e^{- \frac{\mu}{s^*} \theta}\\
& \leq \dots \\
& \leq (f(x^0) - f(x^*)) e^{-T\frac{\mu}{s^*} \theta}\\
& = \epsilon + \frac{4(1-\theta)(\mu\t{\kappa}-1)}{\left(2-\mu\t{\kappa}-\theta\right)^2} (f(x^*) - f(\overline{x}^*)) 
\end{align*}
where the last equality follows by our choice of 
\begin{align*}
T = \frac{\sqrt{ss^*}}{\theta} \log \frac{f(x^0) - f(x^*)}{B}
\end{align*}
and replacing $\mu = \sqrt{\frac{s^*}{s}}$.
This is a contradiction.
\end{prevproof}

\subsection{Corollaries of Theorem~\ref{local_theorem}}

The first corollary states that in the ``noiseless'' case (i.e. when the target solution is globally optimal), the returned solution can reach arbitrarily close to the target solution:
\begin{corollary}[Noiseless case]
If $\widetilde{\kappa}\sqrt{\frac{s^*}{s}} < 2$ and
$x^*$ is a globally optimal solution,
i.e.
$f(x^*) = \underset{z}{\min}\ f(z)$, Algorithm~\ref{local} returns a solution with
\begin{align*}
f(x) \leq f(x^*) + \epsilon
\end{align*}
in $O\left(\frac{\sqrt{ss^*}}{2-\widetilde{\kappa}\sqrt{\frac{s^*}{s}}} \log \frac{f(x^0) - f(x^*)}{\epsilon}\right)$ iterations.
\end{corollary}
\begin{proof}
We apply Theorem~\ref{local_theorem} with $\theta = \frac{1}{2}\left(2 - \widetilde{\kappa}\sqrt{\frac{s^*}{s}}\right)$.
\end{proof}
The following result is in the usual form of sparse recovery results, which provide a bound on $\left\Vert x-x^*\right\Vert_2$
given a RIP constant upper bound. It provides a tradeoff between the RIP constant and
the sparsity of the returned solution.
\begin{corollary}[$\ell_2$ solution recovery]
Given any parameters $\epsilon > 0$ and $0<\theta<1$, the returned solution $x$ 
of Algorithm~\ref{local}
will satisfy
\begin{align*}
\left\Vert x - x^*\right\Vert_2^2 \leq \epsilon + C \left(f(x) - \underset{z}{\min}\ f(z)\right)
\end{align*}
as long as
\begin{align*}
\delta_{s+s^*} < \frac{(2-\theta)\sqrt{\frac{s}{s^*}} - 1}{(2-\theta)\sqrt{\frac{s}{s^*}}+1}
\end{align*}
where $C$ is a constant that depends only on $\theta$, $\delta_{s+s^*}$, and $\frac{s}{s^*}$.
\label{l2}
\end{corollary}
In particular, for $s=s^*$, the above lemma implies recovery under the condition
$\delta_{2s^*} < \frac{1}{3}$.

\section{Lower Bounds}

\subsection{$\Omega(s^*\kappa)$ lower bound due to \cite{FKT15}}
\label{sec:lower_bound}
In Appendix B of \cite{FKT15} a matrix $A\in\mathbb{R}^{m\times n}$ and a vector $b\in\mathbb{R}^m$ are constructed
and let us define $f(x) = \frac{1}{2}\left\Vert Ax-b\right\Vert_2^2$.
If we let $\overline{S^*}=\{1,\dots,n-2\}$ and $S^*=\{n-1,n\}$, then $f$ has the property that
\begin{align*}
\underset{\mathrm{supp}(x)\subseteq S^*}{\mathrm{min}}\ f(x) =  \underset{\mathrm{supp}(x)\subseteq \overline{S^*}}{\mathrm{min}}\ f(x) = 0
\end{align*}
but for any $S\subset \overline{S^*}$,
\begin{align*}
\underset{\mathrm{supp}(x)\subseteq S}{\mathrm{min}}\ f(x) > 0
\end{align*}
Furthermore, for any $S\subset \overline{S^*}$ 
and $x = \underset{\mathrm{supp}(x)\subseteq S}{\mathrm{argmin}}\ f(x)$, it is true that
\begin{align*}
\underset{i\in S^*}{\max}\ \left|\nabla_i f(x)\right| < 
\underset{i\in \overline{S^*}\backslash S}{\min}\ \left|\nabla_i f(x) \right|
\end{align*}
This means that for any algorithm with an OMP-like criterion like Orthogonal Matching Pursuit, Orthogonal Matching Pursuit with Replacement, 
Iterative Hard Thresholding, and Partial Hard Thresholding, 
if the initial solution
does not have an intersection with $S^*$, then it will never have, therefore implying
that the sparsity returned by the algorithm is $|S| = n-2 = \Omega(n)$. As for this construction
$\kappa = \frac{\rho_n^+}{\rho_n^-} = O\left(n\right)$, there exists a constant $c$ such that
the sparsity of the returned solution cannot be less than
$c s^*{\kappa}$, since $s^*\kappa = O(n) = O(\left|S\right|)$.
Therefore none of these algorithms can improve the bound $O(s^*\kappa)$ of Theorem~\ref{reg_theorem} by more than a constant factor.
This example also applies to ARHT and Exhaustive Local Search.

It seems difficult to get past this example and achieve sparsity $s=O(s^* \kappa^{1-\delta})$ for some $\delta > 0$.
We conjecture that there might be a way to turn the above example into an inapproximability result:
\begin{conjecture}
\label{conjecture}
For any $\delta > 0$,
there is no polynomial time algorithm that 
given 
a matrix $A\in\mathbb{R}^{m\times n}$, a vector $b\in\mathbb{R}^m$,
a target sparsity $s^*\geq 1$, and a desired accuracy $\epsilon > 0$,
returns an $s=O(s^*\kappa_{s+s^*}^{1-\delta})$-sparse solution $x$ such that $\left\Vert Ax - b\right\Vert_2^2 \leq \underset{\left\Vert x^*\right\Vert_0 \leq s^*}{\min} 
\left\Vert Ax^*-b\right\Vert_2^2 + \epsilon$, if such a solution exists.
\end{conjecture}

\subsection{$\Omega(s^*\kappa^2)$ lower bound for OMPR}
\label{OMPR_lower_bound}
The following lemma shows that, without regularization, OMPR requires sparsity $\Omega(s^*\kappa^2)$ in general, and therefore the sparsity upper bound is tight.
We assume that the algorithm is run for a fixed $T$ iterations, even when the solution stops improving, for a clearer presentation.
\begin{lemma}
There is a function $f(x) = \frac{1}{2} \left\Vert Ax - b\right\Vert_2^2$ where $A\in\mathbb{R}^{n\times n} $ and $b\in\mathbb{R}^n$ and a target solution $x^*$ of $f$ with sparsity $s^*$, as well as a set $S\subseteq [n]$
with $|S| =\Theta(s^*\kappa^2)$ such that OMPR initialized with support set $S$ returns a solution $x$ with $f(x) = f(x^*) + \Theta(s^*\kappa^2) $.
\end{lemma}
\begin{proof}
Without loss of generality we assume that $\kappa$ is an even integer and 
set $n = s^*\left(1 + \kappa + \kappa^2\right)$. We then partition $[n]$ into three intervals $I_1 = [1,s^*]$, $I_2=[s^*+1,s^*(1+\kappa)]$, $I_3=[s^*(1+\kappa)+1,s^*(1+\kappa+\kappa^2)]$.
We define the diagonal matrix $A\in\mathbb{R}^{n\times n}$ such that 
\begin{align*}
A_{ii} = \begin{cases}
1 & \text{if $i\in I_1$}\\
\sqrt{\kappa} & \text{if $i\in I_2$}\\
1 & \text{if $i\in I_3$}
\end{cases}
\end{align*}
and vector $b\in\mathbb{R}^n$ such that
\begin{align*}
b_i = \begin{cases}
\kappa\sqrt{1- 4\delta} & \text{if $i\in I_1$}\\
\sqrt{\kappa}\sqrt{1 - 2\delta} & \text{if $i\in I_2$}\\
1 & \text{if $i\in I_3$}
\end{cases}
\end{align*}
where $\delta > 0$ is a sufficiently small scalar used to avoid ties in the steps of the algorithm.
The target solution is defined as
\begin{align*}
x_i^* = \begin{cases}
\kappa(1 - 4 \delta) & \text{if $i\in I_1$}\\
0 & \text{if $i\in I_2\cup I_3$}
\end{cases}
\end{align*}
and its value is $f(x^*) = s^*\kappa^2 (1-\delta)$.
Now consider any initial support set $S^0 \subset I_3$ such that $|S^0| = s^* \kappa^2 / 2$. The initial solution will then be 
\begin{align*}
x_i^0 = \begin{cases}
0 & \text{if $i\in I_1\cup I_2 \cup I_3\backslash S^0$}\\
1 & \text{if $i\in S^0$}
\end{cases}
\end{align*}
and its value $f(x^0) = s^*\kappa^2 \left(\frac{5}{4} - 3\delta \right) = f(x^*) + \Theta(s^*\kappa^2)$. The gradient at $x^0$ is
\begin{align*}
\nabla_i f(x^0) = \begin{cases}
-\kappa\sqrt{1-4\delta} & \text{if $i\in I_1$}\\
-\kappa\sqrt{1-2\delta} & \text{if $i\in I_2$}\\
-1 & \text{if $i\in I_3\backslash S^0$}\\
0 & \text{if $i\in S^0$}
\end{cases}
\end{align*}
therefore the algorithm will pick $S^1 = S^0\cup \{i^0\}\backslash\{j^0\}$ for some $i^0\in I_2$ and some $j^0\in S^0$, since the gradient entries in $I_2$ have the largest magnitude among those in $[n]$.
The new solution will be 
\begin{align*}
x_i^1 = \begin{cases}
0 & \text{if $i\in I_1\cup I_2 \cup I_3\backslash S^1$}\\
\sqrt{1-2\delta} & \text{if $i = i^0$}\\
1 & \text{if $i\in S^1 \backslash \{i^0\}$}
\end{cases}
\end{align*}
with value
$f(x^1) = s^* \kappa^2\left(\frac{5}{4} - 3\delta\right) - \frac{1}{2} (\kappa(1-2\delta) - 1)$
and gradient
\begin{align*}
\nabla_i f(x^1) = \begin{cases}
-\kappa\sqrt{1-4\delta} & \text{if $i\in I_1$}\\
-\kappa\sqrt{1-2\delta} & \text{if $i\in I_2 \backslash S^1$}\\
-1 & \text{if $i\in I_3\backslash S^1$}\\
0 & \text{if $i\in S^1$}
\end{cases}
\end{align*}
and therefore the algorithm will pick $S^2 = S^1\cup \{i^1\}\backslash\{i^0\}$ 
for some $i^1\in I_2$. $i^0$ will be the one to be removed from $S^1$ because $x_{i^0}$ has the smallest magnitude out of all entries in $S^1$.
Continuing this process, the algorithm will always have $S^t \cap I_2 = 1$ and $S^t \cap I_3 = |S^t|-1$, and so
$f(x^t) = s^* \kappa^2 \left(\frac{5}{4} - 3\delta\right) - \frac{1}{2} (\kappa(1-2\delta) - 1) = f(x^*) + \Theta(s^* \kappa^2)$
for $t\geq 1$.
\end{proof}

\section{Experiments}
\label{experiment_section}
\subsection{Overview}

In this section we evaluate the training performance of different algorithms in the tasks of Linear Regression and Logistic Regression.
More specifically, for each algorithm we are interested in how  
the \emph{loss} over the training set 
(the quality of 
the solution)
evolves as a function of the
the \emph{sparsity} of the solution, i.e. the number of non-zeros.

The algorithms that we will consider are \emph{LASSO}, \emph{Orthogonal Matching Pursuit (OMP)}, \emph{Orthogonal Matching Pursuit with Replacement (OMPR)}, 
\emph{Adaptively Regularized Hard Thresholding (ARHT)} (Algorithm~\ref{local_reg}),
and \emph{Exhaustive Local Search} (Algorithm~\ref{exlocal}).
We run our experiments on publicly available regression and binary classification datasets, out of which we have presented
those on which the 
algorithms have significantly different performance between each other. In some of the other datasets that we tested, we 
observed that all algorithms had similar performance.
The results are presented in Figures~\ref{fig1}, \ref{fig2}, \ref{fig3}, \ref{fig4}.
Another relevant class of algorithms that we considered was \emph{$\ell_p$ Appproximate Message Passing} algorithms~\cite{donoho2009message,zheng2017does}.
Brief experiments showed its performance in terms of sparsity for $p \leq 0.5$ to be
promising (on par with OMPR and ARHT although these had much faster runtimes), however a detailed comparison is left for future work.

In both types of objectives (linear and logistic) we include an intercept term, which is present in all solutions (i.e. it is always counted as $+1$ in the sparsity
of the solution). 
For consistency, all greedy algorithms (OMPR, ARHT, Exhaustive Local Search) are initialized with the OMP solution of the same sparsity.

The experiments 
make it clear that Exhaustive Local Search outperforms the other algorithms. However, ARHT also has 
promising performance
and it might be preferred because of better computational efficiency. As a general conclusion, however, both Exhaustive Local Search and ARHT offer
an advantage compared to OMP and OMPR.
As a limitation, we observe that ARHT has inconsistent performance in some cases, oscillating between the Exhaustive Local Search and OMPR solutions.

\begin{figure}[H]
\vskip 0.2in
\begin{center}
	\includegraphics[width=0.8\textwidth]{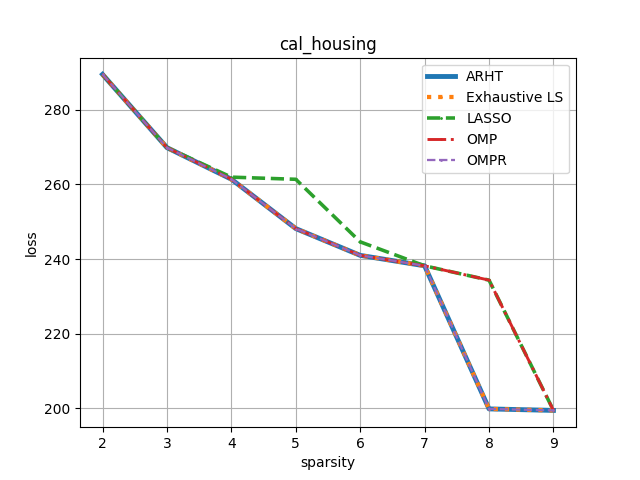}
	\includegraphics[width=0.8\columnwidth]{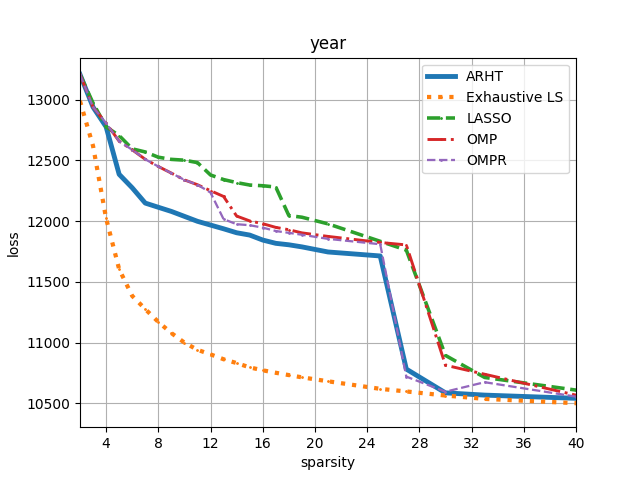}
\caption{Comparison of different algorithms in the Regression datasets \emph{cal\_housing} and \emph{year} using the Linear Regression loss.}
		\label{fig1}
		\end{center}
		\vskip -0.2in
\end{figure}
\begin{figure}[H]
\vskip 0.2in
\begin{center}
	\includegraphics[width=0.8\columnwidth]{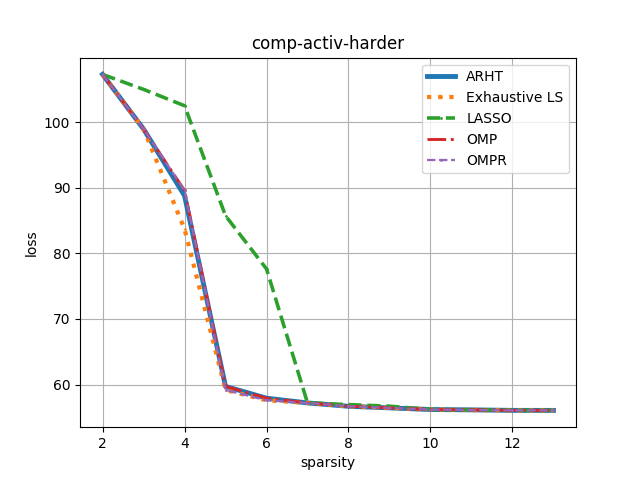}
	\includegraphics[width=0.8\columnwidth]{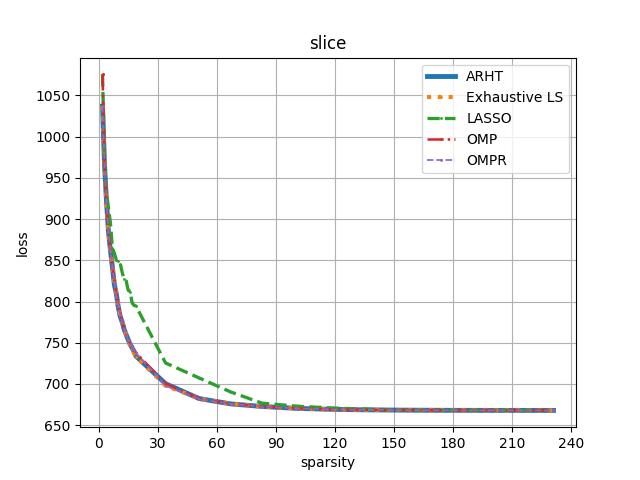}
\caption{Comparison of different algorithms in the Regression datasets \emph{comp-activ-harder} and \emph{slice} using the Linear Regression loss.}
		\label{fig2}
		\end{center}
		\vskip -0.2in
\end{figure}
\begin{figure}[H]
\vskip 0.2in
\begin{center}
\centerline{\includegraphics[width=0.8\columnwidth]{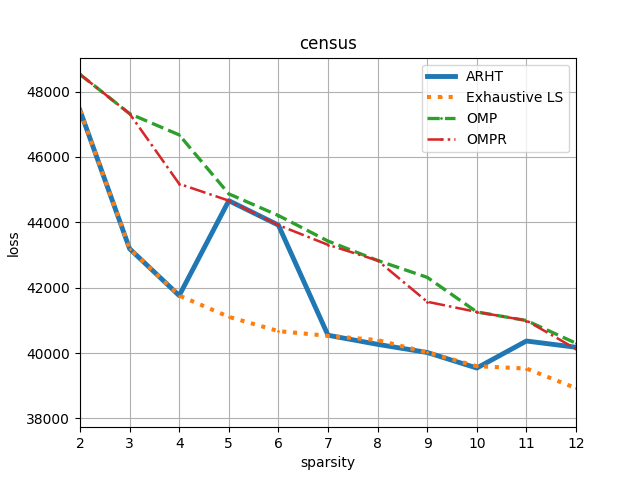}}
\centerline{\includegraphics[width=0.8\columnwidth]{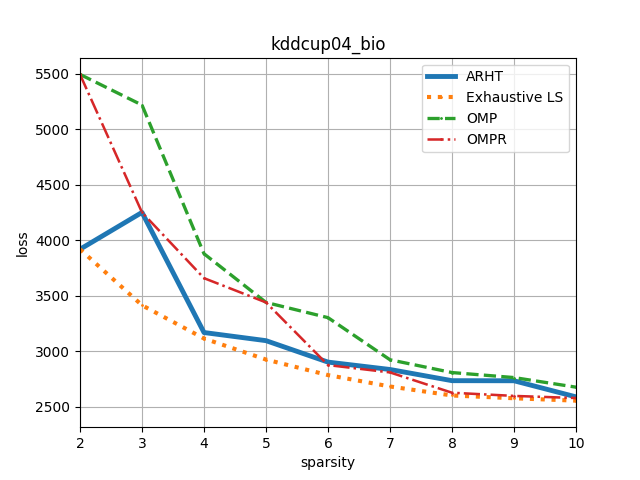}}
\caption{Comparison of different algorithms in the Binary classification datasets \emph{census} and \emph{kddcup04\_bio} using the Logistic Regression loss.}
		\label{fig3}
		\end{center}
		\vskip -0.2in
\end{figure}
\begin{figure}[H]
\vskip 0.2in
\begin{center}
\centerline{\includegraphics[width=0.8\columnwidth]{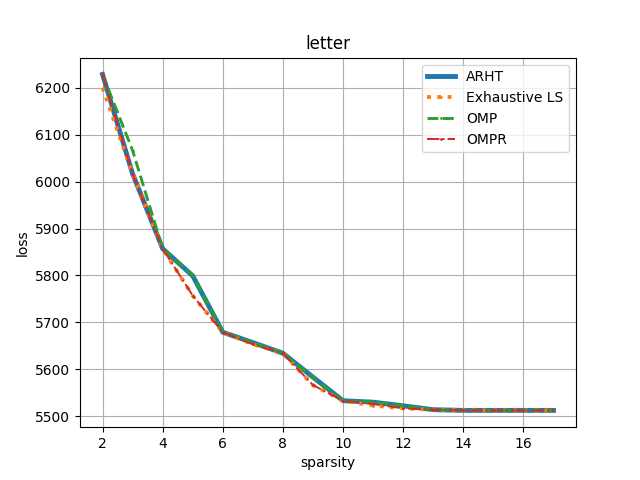}}
\centerline{\includegraphics[width=0.8\columnwidth]{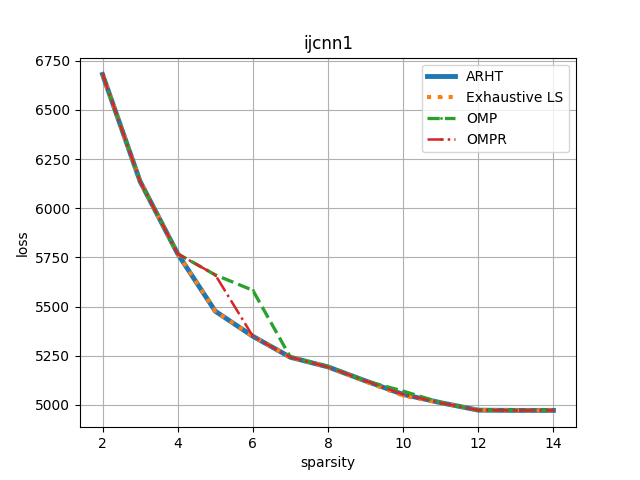}}
\caption{Comparison of different algorithms in the Binary classification datasets \emph{letter} and \emph{ijcnn1} using the Logistic Regression loss.}
		\label{fig4}
		\end{center}
		\vskip -0.2in
\end{figure}

For experimental evaluation we used well known and publicly available datasets. Their names and basic properties are outlined in Table~\ref{dataset_info}.
\begin{table}[H]
\caption{\label{dataset_info}Datasets used for experimental evaluation. The columns are the dataset name,
the number of examples $m$, and the number of features $n$. 
The datasets can be downloaded \href{https://drive.google.com/open?id=1RDu2d46qGLI77AzliBQleSsB5WwF83TF}{here}.}
\vskip 0.15in
\begin{center}
\begin{small}
\begin{sc}
\begin{tabular}{lccc}
\toprule
Name & $n$ & $d$ & problem\\
\midrule
kddcup04\_bio & 145750 & 74 & binary\\
cal\_housing & 20639 & 8 & regression\\
census & 299284 & 401 & binary\\
comp-activ-harder & 8191 & 12 & regression\\
ijcnn1 & 24995 & 22 & binary\\
letter & 20000 & 16 & binary\\
slice & 53500 & 384 & regression\\
year & 463715 & 90 & regression\\
\bottomrule
\end{tabular}
\end{sc}
\end{small}
\end{center}
\vskip -0.1in
\end{table}
\subsection{Setup details}
\subsubsection{Basic Definitions}
The two quantities that take part in our experiments are the \emph{sparsity} and the \emph{loss} of a particular solution.
We have already defined and discussed the former at length. The latter refers to the training loss for the problems of Linear Regression and Logistic Regression.
We let $m$ denote the number of examples and $n$ the number of features in each example.

In the \emph{Linear Regression} task we are given the dataset $(A,b)$, where $A\in\mathbb{R}^{m\times n}$, $b\in\mathbb{R}^m$.
The columns of $A$ correspond to features and the rows to examples.
The \emph{($\ell_2$ Linear Regression) loss} of a solution $x\in\mathbb{R}^n$ is defined as $\mathrm{\ell_2\_loss}(x) = \frac{1}{2}\left\Vert Ax - b\right\Vert_2^2$.

In the \emph{Logistic Regression} task we are given the dataset $(A,b)$, where $A\in\mathbb{R}^{m\times n}$, $b\in\{0,1\}^m$.
The columns of $A$ correspond to features and the rows to examples.
The \emph{(Logistic Regression) loss} of a solution $x\in\mathbb{R}^n$ is defined as 
$\mathrm{logistic\_loss}(x) = \sum\limits_{i\in[m]} \left(- b_i \log \sigma (Ax)_i - (1-b_i) \log (1-\sigma (Ax)_i)\right)$, where 
$\sigma:\mathbb{R}\rightarrow\mathbb{R}$ defined as $\sigma(t) = \frac{1}{1+e^{-t}}$ is the sigmoid function.

\subsubsection{Data Pre-processing}
We apply a very basic form of pre-processing to the data. More specifically, we 
use one-hot encoding to turn categorical features into numerical ones. Then, we discard any examples with missing data so that all the entries of $A$ are defined.
We also augment the matrix $A$ with an extra all-ones column (i.e. $\vec{1}$) in order to encode the constant ($y$-intercept) term into $A$,
and we scale all the columns of $A$ so that their $\ell_2$ norm is $1$.
Finally, for the case of ARHT we further augment $A$ in order to encode the regularizer as well. We do this by adding an identity matrix as extra rows.
In other words, $A\leftarrow \begin{pmatrix}A\\I\end{pmatrix}$ and $b\leftarrow \begin{pmatrix}b\\\vec{0}\end{pmatrix}$.

\subsection{Implementation details}
The code has been implemented in \emph{python3}, with libraries \emph{numpy}, \emph{sklearn}, and \emph{scipy}.

\subsubsection{Inner Optimization Problem}
All the algorithms except for LASSO rely on an inner optimization routine in a restricted subset of coordinates in each step.
The inner optimization problem consists of solving a standard Linear Regression or Logistic Regression problem
using only a submatrix of $A$ defined by a subset of $s$ of its columns.
For that, we use \emph{LinearRegression} and \emph{LogisticRegression} from \emph{sklearn.linear\_model}.
For Logistic Regression we used an LBFGS solver with $1000$ iterations.

\subsubsection{Overall Algorithm}
The LASSO solver we used is \emph{Lasso} from \emph{sklearn.linear\_model} with $1000$ iterations.
As LASSO is not tuned in terms of a required sparsity $s$, but rather in terms of the regularization parameter $\alpha$,
for each sparsity level we applied binary search on $\alpha$ 
in order to find a parameter $\alpha$ that gives the required sparsity.

For ARHT, we used a fixed number of $20$ iterations at Line 5 of Algorithm~\ref{local_reg}. In Line 19 of Algorithm~\ref{arls}
we slightly weaken the progress condition to
\begin{align}
g_{R^t}(x^t) - g_{R^t}(x^{t+1}) \geq \frac{10^{-3}}{s} \left(g_{R^t}(x^t) - \mathrm{opt}\right) \label{prog_cond}
\end{align}
so that it does not depend
Furthermore, we do not perform a fixed number of iterations. Instead, 
we use a stopping criterion: If the progress condition (\ref{prog_cond}) is not met and at least half the elements in $x^t$ have already been unregularized, 
i.e. $\left|S^t\backslash R^t\right| \geq \frac{1}{2} \left|S^t\right|$,
then we stop. If a desirable solution has not been found, it means that this might be an unsuccessful run,
and early termination can be used to detect such runs early and re-start, thus improving the runtime.
The routine which samples an index $i$ proportional to $x_i^2$ was implementing by a standard sampling method that uses binary search on $i$ and flips a random
coin at each step. This requires computation of interval sums of $x_i^2$, which is done by computing partial sums.

\bibliography{references}
\bibliographystyle{alpha}

\appendix
\section{Deferred Proofs}

\subsection{Proof of Lemma~\ref{reg_condition}} 
\label{proof_lemma_reg_condition}

\begin{proof}
$\Phi^t$ is a quadratic restricted on $R^t$
\begin{align*}
& \Phi(y) - \Phi(x) - \nabla \Phi(x)^T(y-x)\\
& = \frac{\rho_2^+}{2} \left(\left\Vert y_{R^t}\right\Vert_2^2 
	-\left\Vert x_{R^t}\right\Vert_2^2 
	- 2 x_{R^t}^T(y_{R^t}-x_{R^t})\right)\\
& = \frac{\rho_2^+}{2} \left\Vert y_{R^t}-x_{R^t}\right\Vert_2^2 \\
& \in \left[0,\frac{\rho_2^+}{2} \left\Vert y-x\right\Vert_2^2\right]
\end{align*}
and so for any $x,y$ with $|\mathrm{supp}(y-x)| \leq s+s^*$ (resp. $|\mathrm{supp}(y-x)| \leq 1$)  we have
\begin{align*}
& g(y) - g(x) - \nabla g(x)^T (y-x)\\
& = f(y) - f(x) - \nabla f(x)^T (y-x) + \Phi(y) - \Phi(x) - \nabla \Phi(x)^T(y-x)\\
& \geq \frac{\rho_{s+s^*}^-}{2} \left\Vert y-x\right\Vert_2^2 \text{(resp.} \leq \rho_2^+ \left\Vert y-x\right\Vert_2^2 \text{)}
\end{align*}
\end{proof}

\subsection{Proof of Lemma~\ref{recurrence}} 
\label{proof_lemma_recurrence}

\begin{proof}
By definition, and setting $\tau=\frac{1}{s}$, for each Type 1 iteration we have 
\begin{align*}
& g(x^t) - g(x^{t+1}) \geq \tau \left(g(x^t) - f(x^*)\right) \\
\Rightarrow  & 
g(x^{t+1}) - f(x^*) \leq (1-\tau) (g(x^{t}) - f(x^*))
\end{align*}
and in each Type 2 iteration we have
\begin{align*}
g(x^{t+1}) - f(x^*) \leq g(x^{t}) - f(x^*)
\end{align*}
(since $g$ can only decrease when unregularizing),
therefore 
\begin{align*}
 f(x^T) - f(x^*) 
& \leq g(x^{T}) - f(x^*) \\
& \leq (1-\tau)^{T_1} (g(x^{0}) - f(x^*)) \\
& \leq e^{-\tau T_1} (g(x^0) - f(x^*)) \\
& \leq \epsilon
\end{align*}
where we used the fact that 
$T_1 = \frac{1}{\tau} \log \frac{g(x^0) - f(x^*)}{\epsilon}$.
\end{proof}

\subsection{Proof of Lemma~\ref{lemma_strongconv}} 
\label{proof_lemma_strongconv}
\begin{proof}
We have 
\begin{align*}
 \left(\sqrt{f(x^t) - f(\widetilde{x}^t)} + \sqrt{f(x^*) - f(\widetilde{x}^t)}\right)^2 
& \geq \frac{\rho^-}{2} \left(\left\Vert x^t - \widetilde{x}^t\right\Vert_2 + \left\Vert x^* - \widetilde{x}^t\right\Vert_2\right)^2\\
& \geq \frac{\rho^-}{2} \left\Vert x^t - x^*\right\Vert_2^2\\
& \geq \frac{\rho^-}{2} \left(\left\Vert x^*_{S^*\backslash S^t}\right\Vert_2^2 + \left\Vert x_{S^t\backslash S^*}^t\right\Vert_2^2\right)
\end{align*}
where the first inequality follows by applying strong convexity to lower bound
$f(x^t) - f(\widetilde{x}^t)$ and $f(x^*) - f(\widetilde{x}^t)$
combined with the fact that by definition of $\widetilde{x}^t$, $\nabla_{S^t\cup S^*} f(\widetilde{x}^t) = \vec{0}$,
and the second is a triangle inequality.
\end{proof}

\end{document}